\documentclass{article}

\usepackage{microtype}
\usepackage{graphicx}
\usepackage{subcaption}
\usepackage{booktabs} 

\usepackage{hyperref}

\usepackage[preprint]{icml2026}

\usepackage{amsmath}
\usepackage{amssymb}
\usepackage{mathtools}
\usepackage{amsthm}
\usepackage[table]{xcolor}  
\usepackage{colortbl}       
\usepackage{array}   

\usepackage[capitalize,noabbrev]{cleveref}

\theoremstyle{plain}
\newtheorem{theorem}{Theorem}[section]
\newtheorem{proposition}[theorem]{Proposition}
\newtheorem{lemma}[theorem]{Lemma}
\newtheorem{corollary}[theorem]{Corollary}
\theoremstyle{definition}
\newtheorem{definition}[theorem]{Definition}

\theoremstyle{remark}

\usepackage{tikz}
\usetikzlibrary{positioning,arrows.meta,shapes.misc}
\usepackage{tikz}
\usepackage{multirow}
\usepackage[textsize=tiny]{todonotes}

\icmltitlerunning{}

\begin{document}

\twocolumn[
  \icmltitle{Minimizing Mismatch Risk: A Prototype-Based Routing Framework for Zero-shot LLM-generated Text Detection}

  \icmlsetsymbol{visiting}{\dagger}

  \begin{icmlauthorlist}
    \icmlauthor{Ke Sun}{sch}
    \icmlauthor{Guangsheng Bao}{sch}
    \icmlauthor{Han Cui}{sch}
    \icmlauthor{Yue Zhang*}{sch}
  \end{icmlauthorlist}

  \icmlaffiliation{sch}{School of Engineering, Westlake University, China}

  \icmlcorrespondingauthor{Yue Zhang}{zhangyue@westlake.edu.cn} 

  \icmlkeywords{Machine Learning, LLM Detection, Zero-shot Learning, Prototype-Based Routing}

  \vskip 0.3in
]



\printAffiliationsAndNotice{}  

\begin{abstract}
Zero-shot methods detect LLM-generated text by computing statistical signatures using a surrogate model. Existing approaches typically employ a fixed surrogate for all inputs regardless of the unknown source. We systematically examine this design and find that detection performance varies substantially depending on surrogate-source alignment. We observe that while no single surrogate achieves optimal performance universally, a well-matched surrogate typically exists within a diverse pool for any given input. This finding transforms robust detection into a routing problem: selecting the most appropriate surrogate for each input. We propose DetectRouter, a prototype-based framework that learns text-detector affinity through two-stage training. The first stage constructs discriminative prototypes from white-box models; the second generalizes to black-box sources by aligning geometric distances with observed detection scores. Experiments on EvoBench and MAGE benchmarks demonstrate consistent improvements across multiple detection criteria and model families.

\end{abstract}
\section{Introduction}

The rapid proliferation of Large Language Models has revolutionized natural language generation, enabling applications ranging from creative writing to code synthesis. However, this capability poses significant risks to information integrity, including the mass production of disinformation, academic dishonesty, and the contamination of web corpora used to train future models~\citep{milano2023large,demszky2023using,anil2024generativeAI}. Consequently, distinguishing between human-written and machine-generated text has become an imperative safety challenge~\citep{yang2024survey,wu2025survey}.

\begin{figure}[t]
    \begin{center}
       \includegraphics[width=1\linewidth]{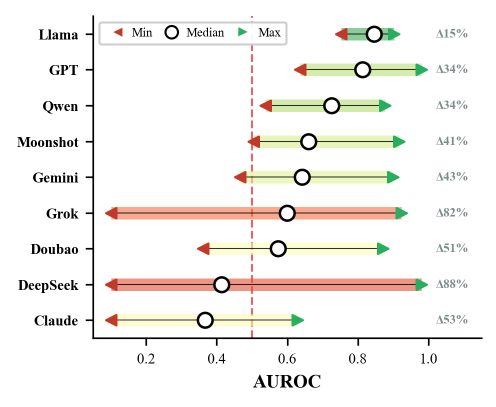}
    \end{center}
       \caption{Detection variance of Fast-DetectGPT across nine black-box LLM families on the MIRAGE benchmark. Each bar spans from the minimum to maximum AUROC achieved by different surrogate models, with circles indicating the median. The percentage on the right denotes the performance gap between the best and worst surrogate for each source.}
    \label{fig:intro}
\end{figure}

Among existing detection approaches, zero-shot methods have gained prominence for their training-free nature. These methods employ a surrogate LLM to compute token-level statistics of the candidate text, such as conditional probability or likelihood curvature, and aggregate these statistics into a detection score that discriminates machine-generated content from human writing~\citep{mitchell2023detectgpt,bao2024fastdetectgpt,sun2026ai}. While such methods achieve strong performance in the white-box setting where the source model is accessible, their effectiveness degrades substantially in real-world black-box scenarios where the generating model is unknown~\citep{fu2025detectanyllm,bao2025glimpse}. This gap presents an urgent challenge that demands systematic investigation.

A key factor underlying this limitation is that all existing zero-shot methods employ a fixed surrogate model for detection across diverse text sources, justified by the distribution matching hypothesis that probability distributions across different LLMs are similar~\citep{mitchell2023detectgpt,bao2024fastdetectgpt}. Our preliminary experiments challenge this assumption. As shown in Figure~\ref{fig:intro}, evaluating Fast-DetectGPT with nine open-source surrogate models on the MIRAGE benchmark~\citep{fu2025detectanyllm} reveals dramatic performance variance: for challenging sources such as DeepSeek and Grok, the gap between the best and worst surrogate exceeds 80\%, spanning from near-random to strong detection. We further quantify this phenomenon in Section~\ref{sec:analysis} and derive the Mismatch Risk Bound, which establishes that the performance gap is upper-bounded by the square root of the KL divergence between source and surrogate distributions. This bound transforms detection into a routing problem: given an input text, we must identify the surrogate that minimizes distributional divergence to the unknown source.

Building on this insight, we propose DetectRouter, a prototype-based routing framework for zero-shot LLM detection. Given a text of unknown origin, DetectRouter first encodes it into a learned metric space, then identifies the nearest prototype among a set of detector-specific representations, and finally routes the text to the corresponding optimal detector for scoring. Since true source distributions are inaccessible for black-box models, we approximate distributional affinity geometrically: we learn an embedding space where distances to detector prototypes reflect underlying distributional proximity. DetectRouter employs a two-stage training strategy. Stage~1 constructs discriminative prototypes using white-box models with known source labels through contrastive learning. Stage~2 generalizes to black-box sources by treating detection scores from the surrogate pool as supervision and minimizing the KL divergence between predicted affinity and observed score distributions. At inference time, the router routes each input to its most compatible detector via nearest-prototype matching.

Extensive experiments on EvoBench and MAGE benchmarks demonstrate that DetectRouter consistently improves all six detection criteria, with relative gains ranging from 5.4\% to 139.4\% and averaging 36.1\% over fixed-surrogate baselines. DetectRouter achieves state-of-the-art performance with 90.85\% average AUROC on EvoBench and 77.92\% on MAGE, outperforming the best baseline by 9.84\% and 4\% points respectively. 

Our contributions are summarized as follows:
\begin{itemize}
    \item We conduct the first systematic analysis of surrogate selection in zero-shot LLM detection, revealing that surrogate-source affinity is the decisive factor for black-box performance.
    \item We propose DetectRouter, a prototype-based routing framework that learns to match input texts to optimal detectors through two-stage training combining discriminative prototype construction with distributional alignment.
    \item Extensive experiments on two benchmarks demonstrate that DetectRouter achieves state-of-the-art performance and serves as a universal enhancement layer for all zero-shot detection methods.
\end{itemize}

\section{Related Work}

\paragraph{Text generation detection.}
Existing methods fall into two categories. Training-based methods fine-tune classifiers with ensemble, contrastive, or adversarial objectives~\cite{abburi2023generative,hu2023radar,guo2024biscope,chen2025imitate,su2023hc3}, but often degrade on unseen models or prompts~\cite{yang2024survey,wu2025survey,sun2025idiosyncrasies}. Statistics-based methods instead treat a surrogate LLM as a scoring function and design detection rules from token probabilities, yielding zero-shot detectors without task-specific data. GLTR uses token-level statistics such as log-likelihood, entropy, and rank histograms~\cite{gehrmann2019gltr}; DetectLLM exploits log-rank scores~\cite{su2023detectllm}; DetectGPT measures log-likelihood curvature under perturbations~\cite{mitchell2023detectgpt}; and Fast-DetectGPT replaces sampling with an efficient analytic approximation~\cite{bao2024fastdetectgpt}. Recent extensions include kernel-based tests~\cite{yu2024dpic,song2025deep} and memorization-based detectors~\cite{guo2024biscope}. Despite these advances, most methods assume a single surrogate suffices; our work instead routes each input to the most suitable detector from a pool.

\paragraph{Routing and origin tracing for LLMs.}
Model routing learns to assign inputs to appropriate models from a pool with varying accuracy-cost trade-offs, enabling efficient inference under budget constraints~\citep{hu2024routerbench,huang2025routereval,jitkrittum2025universal,chuang2024learning}. However, existing approaches primarily treat LLMs as task solvers for QA or reasoning, rather than as detectors for generated text. Origin-tracing methods such as~\citet{li2023origin} and~\citet{fu2025fdllm} build classifiers to identify which LLM produced a given text in black-box scenarios, but their goal is provenance attribution rather than exploiting detector complementarity for improved detection. \citet{sun2025idiosyncrasies} show that different LLMs exhibit highly separable idiosyncratic patterns, enabling source recovery with high accuracy. 

\begin{figure}[t]
    \centering
    \includegraphics[width=1.0\linewidth]{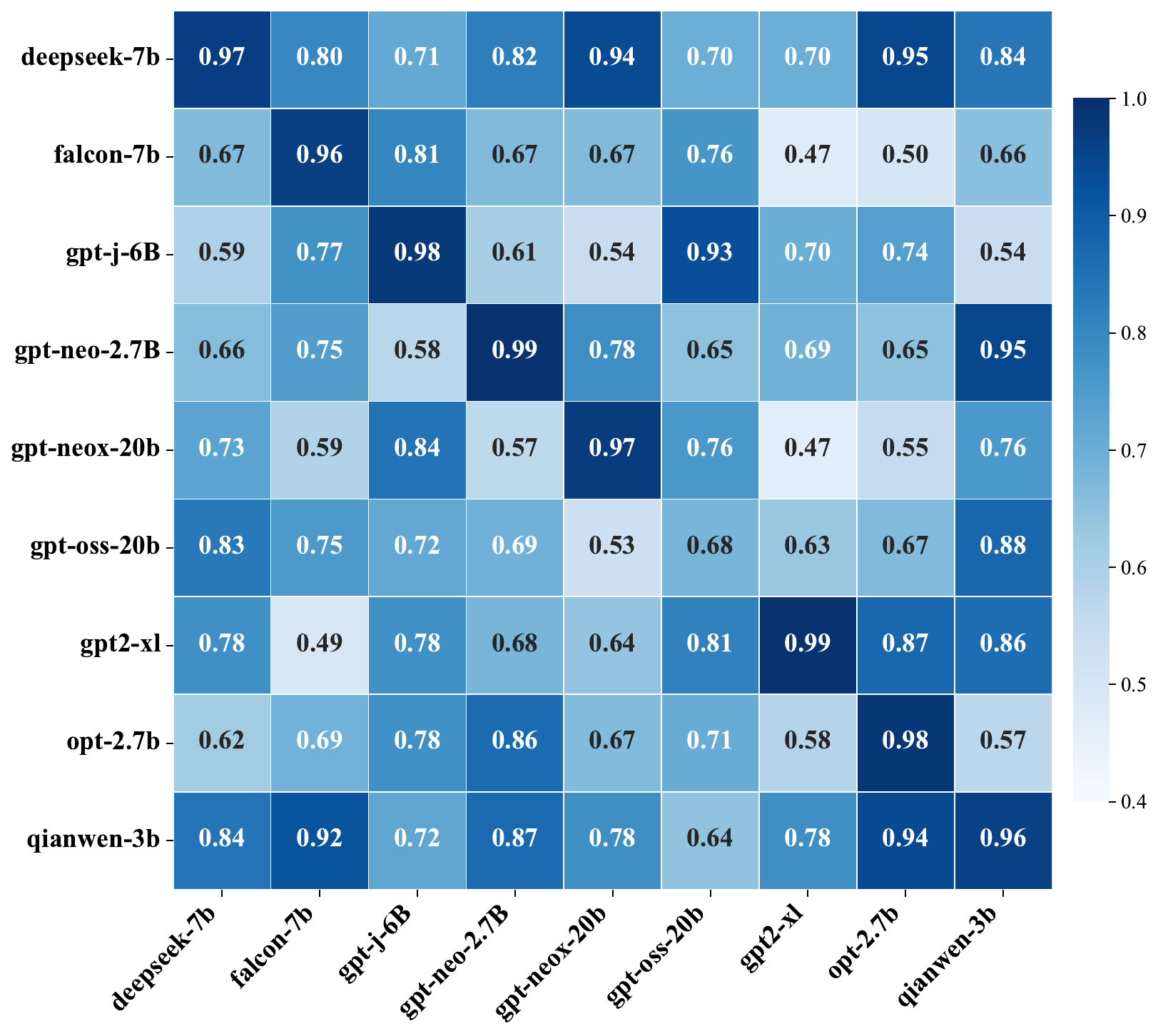}
    \caption{Affinity matrix of detection performance using Fast-DetectGPT. Nine open-source models serve as both generators (rows) and surrogates (columns). Diagonal dominance confirms optimal performance requires source-surrogate alignment, while off-diagonal patterns reveal structured affinity governed by architectural similarity.}
    \label{fig:analysis_confusion}
\end{figure}

\section{Analysis: The Cost of Misalignment}
\label{sec:analysis}

Before presenting our method, we investigate the validity of the Distribution Matching Hypothesis that underpins current zero-shot detection approaches.
Standard methods employ a fixed surrogate model to estimate statistical artifacts, assuming its probability landscape sufficiently approximates diverse source distributions~\citep{mitchell2023detectgpt, bao2024fastdetectgpt}.
While appealing in simplicity, this assumption overlooks \textit{model idiosyncrasies}---distinct statistical fingerprints arising from varying architectures and pre-training corpora~\citep{sun2025idiosyncrasies, yang2024survey}.

\subsection{Experimental Setup}
\label{sec:analysis_setup}

We design cross-evaluation experiments in both white-box and black-box settings. For the white-box setting, we follow the protocol of Fast-DetectGPT~\citep{bao2024fastdetectgpt} and use XSum~\citep{Narayan2018DontGM} as the seed corpus. Nine open-source LLMs ranging from 2.7B to 20B parameters each generate 500 text continuations via nucleus sampling. We then evaluate every model as a surrogate detector on all nine corpora, yielding a $9 \times 9$ cross-evaluation matrix. For the black-box setting, we utilize the MIRAGE benchmark~\citep{fu2025detectanyllm} with proprietary models including GPT-4, Claude, Gemini, and DeepSeek. Complete experimental details are provided in Appendix~\ref{app:analysis_setup}.

\subsection{Empirical Investigation}
\label{sec:analysis_empirical}

\paragraph{White-box Setting.}
Figure~\ref{fig:analysis_confusion} visualizes the pairwise AUC scores for Fast-DetectGPT.
The heatmap reveals two critical patterns.
First, we observe pronounced diagonal dominance, confirming that detectors achieve optimal performance only when source equals surrogate.
Second, the off-diagonal degradation is structured: surrogates sharing architectural lineage, such as GPT-Neo and GPT-J families, maintain high mutual affinity, whereas structurally distinct pairs exhibit sharp performance collapse.
Notably, even among models of comparable scale, architectural differences alone can cause AUC drops exceeding 30 percentage points. We provide affinity matrices for all five detection criteria in Appendix~\ref{app:analysis_white}.

\paragraph{Black-box Setting.}
As illustrated in Figure~\ref{fig:intro}, the performance variance across different surrogate choices is extreme.
For challenging sources like Claude and DeepSeek, AUC spans from near-random levels around 0.2 to near-perfect detection exceeding 0.9, depending solely on surrogate selection.
While the median performance hovers near random baseline, the maximum consistently remains high---confirming that detection capability exists within the open-source community but is fragmented across models.
This fragmentation implies that no single surrogate universally dominates, necessitating adaptive selection based on input characteristics. Extended visualizations including per-source distributions are provided in Appendix~\ref{app:analysis_black}.

\begin{table}[t]
\centering
\caption{Quantifying source-surrogate misalignment across five detection criteria. Matched: source = surrogate; Cross: source $\neq$ surrogate. For black-box setting, Max/Min denote the best/worst surrogate averaged across all generator families.}
\label{tab:analysis_cost}
\small
\setlength{\tabcolsep}{4pt}
\begin{tabular}{l|cc|ccc}
\toprule
& \multicolumn{2}{c|}{White-box / XSum} & \multicolumn{3}{c}{Black-box / MIRAGE} \\
Criterion & Matched & Cross & Mean & Max & Min \\
\midrule
Likelihood & 0.81 & 0.42 & 0.49 & 0.76 & 0.35 \\
Rank & 0.76 & 0.60 & 0.52 & 0.64 & 0.43 \\
LogRank & 0.85 & 0.46 & 0.48 & 0.73 & 0.35 \\
Entropy & 0.55 & 0.71 & 0.59 & 0.74 & 0.41 \\
Fast-DetectGPT & 0.94 & 0.71 & 0.65 & 0.93 & 0.47 \\
\midrule
Average & 0.78 & 0.58 & 0.55 & 0.76 & 0.40 \\
\bottomrule
\end{tabular}
\end{table}

\subsection{Quantitative Analysis}
\label{sec:analysis_quantitative}

Table~\ref{tab:analysis_cost} quantifies the cost of source-surrogate misalignment across five detection criteria.
In the white-box setting, average AUC drops from 0.78 in the matched condition to 0.58 in the cross-model condition, representing a 26\% relative decline.
For Fast-DetectGPT, performance degrades from 0.94 to 0.71---a gap that transforms a near-perfect detector into a mediocre one.

In the black-box setting, variance across surrogate choices is equally striking.
While average performance is only 0.55, the best surrogate achieves 0.76 on average, confirming that optimal surrogates exist for each source.
For Fast-DetectGPT, the spread is particularly pronounced: the best surrogate achieves 0.93 while the worst drops to 0.47.
These findings establish that fixed surrogate selection incurs substantial penalties, yet optimal choices exist within a diverse model pool---motivating a routing mechanism for dynamic surrogate selection.
\subsection{Theoretical Justification}
\label{sec:analysis_theory}

The empirical results highlight a consistent pattern: detection reliability is inversely correlated with the distributional distance between surrogate and source.
To formalize this, we establish the following notation. Let $P_{src}$ denote the probability distribution of the source model that generated the text, and let $P_{sur}$ denote the distribution of the surrogate model used for detection. Given a bounded scoring function $T(x): \mathcal{X} \to \mathbb{R}$ with $|T(x)| \leq B$, the optimal detection signal is $\mu^* = \mathbb{E}_{P_{src}}[T(x)]$, while the proxy signal obtained using the surrogate is $\mu_{proxy} = \mathbb{E}_{P_{sur}}[T(x)]$.

\begin{proposition}[Mismatch Risk Bound]
\label{prop:mismatch_bound}
For any bounded zero-shot detection statistic $T$ with $|T(x)| \leq B$:
\begin{equation}
    |\mu^* - \mu_{proxy}| \leq B \cdot \sqrt{2 D_{KL}(P_{src} \| P_{sur})}
    \label{eq:mismatch_bound}
\end{equation}
where $D_{KL}$ denotes the Kullback-Leibler divergence.
\end{proposition}

The proof proceeds via Total Variation distance and Pinsker's inequality; we provide the complete derivation in Appendix~\ref{app:theory_proof}.
This bound rigorously explains our empirical findings: cross-model degradation reflects inflated KL divergence when architectures differ, and the optimal surrogate is precisely the minimizer of this divergence.
Consequently, robust zero-shot detection transforms into a routing problem: given a pool of surrogate models $\{M_1, \ldots, M_K\}$ with corresponding distributions $\{P_{M_1}, \ldots, P_{M_K}\}$, the optimal strategy is to select the surrogate $M_k$ that minimizes $D_{KL}(P_{src} \| P_{M_k})$ for each input.
This insight forms the foundation of DetectRouter.

\section{Methodology}
\label{sec:method}

\begin{figure*}[t]
    \centering
    \includegraphics[width=1.0\linewidth]{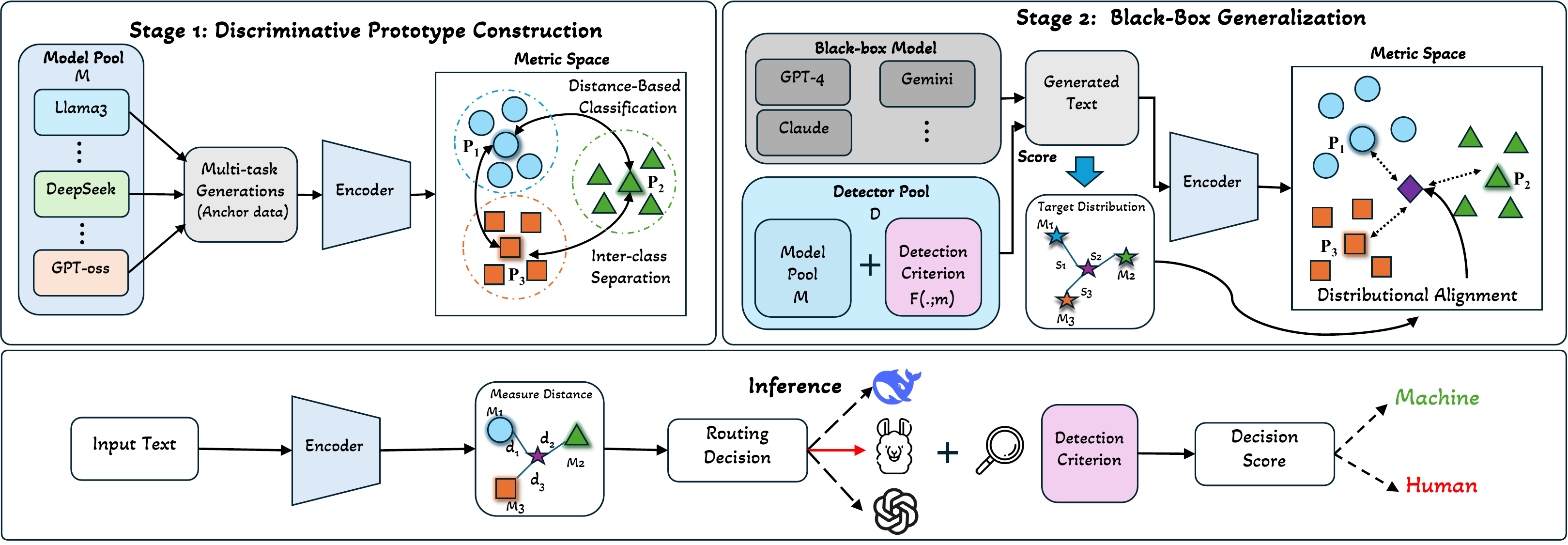}
    \caption{Overview of DetectRouter. Stage~1 constructs discriminative prototypes from white-box models via multi-task generations and distance-based classification. Stage~2 generalizes to black-box sources by aligning geometric distances with detection score distributions. At inference, the router selects the optimal detector based on nearest-prototype affinity.}
    \label{fig:framework}
\end{figure*}

\subsection{Problem Formulation}
\label{sec:method_formulation}

Our analysis in Section~\ref{sec:analysis} suggests that the optimal strategy is to dynamically select a surrogate $M_k$ minimizing $D_{KL}(P_{src} \| P_{M_k})$. This reframes detection as a routing problem as shown in Figure~\ref{fig:framework}: given text $x$ from unknown source $M_{src}$, identify the optimal detector from pool $\mathcal{D} = \{D_1, \ldots, D_N\}$, where each $D_k$ pairs a surrogate $M_k$ with a detection criterion $F(\cdot; M_k)$ such as Likelihood or Fast-DetectGPT:
\begin{equation}
    k^* = \operatorname*{arg\,min}_{k \in \{1,\dots,N\}} D_{KL}(P_{src} \| P_{M_k})
    \label{eq:optimal_routing}
\end{equation}

Since $P_{src}$ is inaccessible for black-box models, we approximate distributional affinity geometrically. We learn an encoder $E_\theta: \mathcal{X} \to \mathbb{R}^d$ and represent each detector $D_k$ with learnable prototypes $\mathcal{P}_k = \{\mathbf{p}_{k,1}, \ldots, \mathbf{p}_{k,K}\}$ capturing the signatures of model family $k$. The routing decision becomes a nearest-prototype query:
\begin{equation}
    k^* = \operatorname*{arg\,min}_{k} \min_{\mathbf{p} \in \mathcal{P}_k} \| E_\theta(x) - \mathbf{p} \|_2
    \label{eq:geometric_routing}
\end{equation}
We employ two-stage training: Stage~1 builds discriminative prototypes from white-box models, while Stage~2 aligns black-box sources via distributional alignment.

\subsection{Stage 1: Discriminative Prototype Construction}
\label{sec:method_stage1}

The first stage establishes a well-structured metric space using $N$ white-box models from the model pool $\mathcal{M}$ where ground-truth source labels are available. A key design consideration is ensuring that prototypes capture intrinsic stylistic fingerprints rather than task-specific artifacts. We construct a diverse anchor dataset $\mathcal{T}_1 = \{(x_j, y_j)\}_{j=1}^{|\mathcal{T}_1|}$ through multi-task generations: each model $M_i$ is tasked with (1). direct generation from seed prompts, (2). text polishing of human drafts, and (3). rewriting of existing passages. This diversity forces the encoder to identify model-specific patterns that persist across varying conditions, yielding prototypes that generalize beyond any single task.

\paragraph{Distance-Based Classification.}
We quantify the affinity between an input $x$, embedded as $\mathbf{z} = E_\theta(x)$, and a detector class $i$ via the distance to its nearest prototype:
\begin{equation}
    d_i(\mathbf{z}) = \min_{k \in \{1, \ldots, K\}} \|\mathbf{z} - \mathbf{p}_{i,k}\|_2
    \label{eq:prototype_distance}
\end{equation}
These geometric distances are transformed into classification probabilities using a temperature-scaled softmax, where a smaller distance implies higher affinity:
\begin{equation}
    P(y=i \mid \mathbf{z}) = \frac{\exp(-d_i(\mathbf{z}) / \tau)}{\sum_{j=1}^N \exp(-d_j(\mathbf{z}) / \tau)}
    \label{eq:class_probability}
\end{equation}

\paragraph{Training Objectives.}
The encoder and prototypes are optimized jointly using a composite objective. We minimize the cross-entropy loss $\mathcal{L}_{\text{CE}}$ to cluster embeddings around their ground-truth prototypes:
\begin{equation}
    \mathcal{L}_{\text{CE}} = -\frac{1}{|\mathcal{T}_1|} \sum_{(x,y) \in \mathcal{T}_1} \log P(y \mid E_\theta(x))
\end{equation}
To further structure the metric space, we introduce a margin-based separation loss $\mathcal{L}_{\text{sep}}$ that enforces inter-class separation by penalizing prototype pairs from different classes closer than margin $m$:
\begin{equation}
    \mathcal{L}_{\text{sep}} = \frac{1}{Z} \sum_{i \neq j} \sum_{k, l} \max\left(0, m - \|\mathbf{p}_{i,k} - \mathbf{p}_{j,l}\|_2\right)^2
\end{equation}
A normalization term $\mathcal{L}_{\text{norm}}=\sum_{i,k} \|\mathbf{p}_{i,k}\|_2^2$ constrains prototype magnitude for training stability. The complete Stage~1 objective is:
\begin{equation}
    \mathcal{L}_1 = \mathcal{L}_{\text{CE}} + \lambda_{\text{sep}} \mathcal{L}_{\text{sep}} + \lambda_{\text{norm}} \mathcal{L}_{\text{norm}}
\end{equation}
Upon convergence, this yields a discriminative embedding space characterized by high intra-class compactness and clear inter-class separation.

\subsection{Stage 2: Black-Box Generalization via Distributional Alignment}
\label{sec:method_stage2}

While Stage~1 builds the coordinate system, it relies on explicit labels. Real-world scenarios involve unknown black-box models such as GPT-4 where labels are inaccessible. To bridge this gap, we leverage the vector of detection scores produced by the detector pool $\mathcal{D}$ as an empirical proxy for distributional alignment.

\paragraph{Target Distribution Construction.}
For a black-box text $x$, we observe a score vector $\mathbf{s} = [s_1, \ldots, s_N]$ from the detector pool. These scores reflect the text's statistical proximity to each known model family. Instead of regressing absolute scores, which can be unstable due to scale differences, we normalize them into a target distribution $Q$ via temperature-scaled softmax:
\begin{equation}
    Q_i = \frac{\exp(s_i / T)}{\sum_{j=1}^N \exp(s_j / T)}
    \label{eq:target_dist}
\end{equation}
Simultaneously, the router outputs a predicted affinity distribution $P$ derived from the geometric distances in the prototype space as defined in Eq.~\ref{eq:class_probability}.

\paragraph{KL-Divergence Optimization.}
We formulate the learning objective as a distribution matching problem, minimizing the Kullback-Leibler divergence between the target distribution $Q$ and the predicted geometric distribution $P$:
\begin{equation}
    \mathcal{L}_{\text{KL}} = D_{KL}(Q \| P) = \sum_{i=1}^N Q_i \log \frac{Q_i}{P_i}
    \label{eq:kl_loss}
\end{equation}
This objective explicitly aligns the router's geometric logic with the theoretical goal of minimizing mismatch risk: the router learns to position the black-box embedding such that its distances to prototypes mirror the actual detection affinities.

\paragraph{Prototype Anchoring.}
Unsupervised adaptation risks catastrophic forgetting, where prototypes might drift and destroy the discriminative structure learned in Stage~1. To prevent this, we employ prototype anchoring by freezing a copy of the converged Stage~1 prototypes $\{\mathbf{p}^{\text{anchor}}\}$ and penalizing deviations:
\begin{equation}
    \mathcal{L}_{\text{anchor}} = \sum_{i,k} \|\mathbf{p}_{i,k} - \mathbf{p}_{i,k}^{\text{anchor}}\|_2^2
    \label{eq:anchor_loss}
\end{equation}
The complete Stage~2 objective combines distributional alignment with regularization:
\begin{equation}
    \mathcal{L}_2 = \mathcal{L}_{\text{KL}} + \lambda_{\text{anc}} \mathcal{L}_{\text{anchor}} + \lambda_{\text{sep}} \mathcal{L}_{\text{sep}} + \lambda_{\text{norm}} \mathcal{L}_{\text{norm}}
    \label{eq:stage2_loss}
\end{equation}
This ensures the router generalizes to unseen sources while retaining the robust topology of known models.

\subsection{Inference}
\label{sec:method_inference}

During inference, given an input text $x$, the encoder projects it into the metric space to calculate the distances $\{d_1, \ldots, d_N\}$ against all detector prototypes. The routing decision identifies the detector with the highest affinity by solving $k^* = \arg\min_k d_k$. The input is then routed to the selected detector $D_{k^*}$, which computes the final detection score to classify the text as human-written or machine-generated.

\section{Experiments}

\subsection{Experimental Setup}
\begin{figure*}[t]
    \centering
    \includegraphics[width=1.0\linewidth]{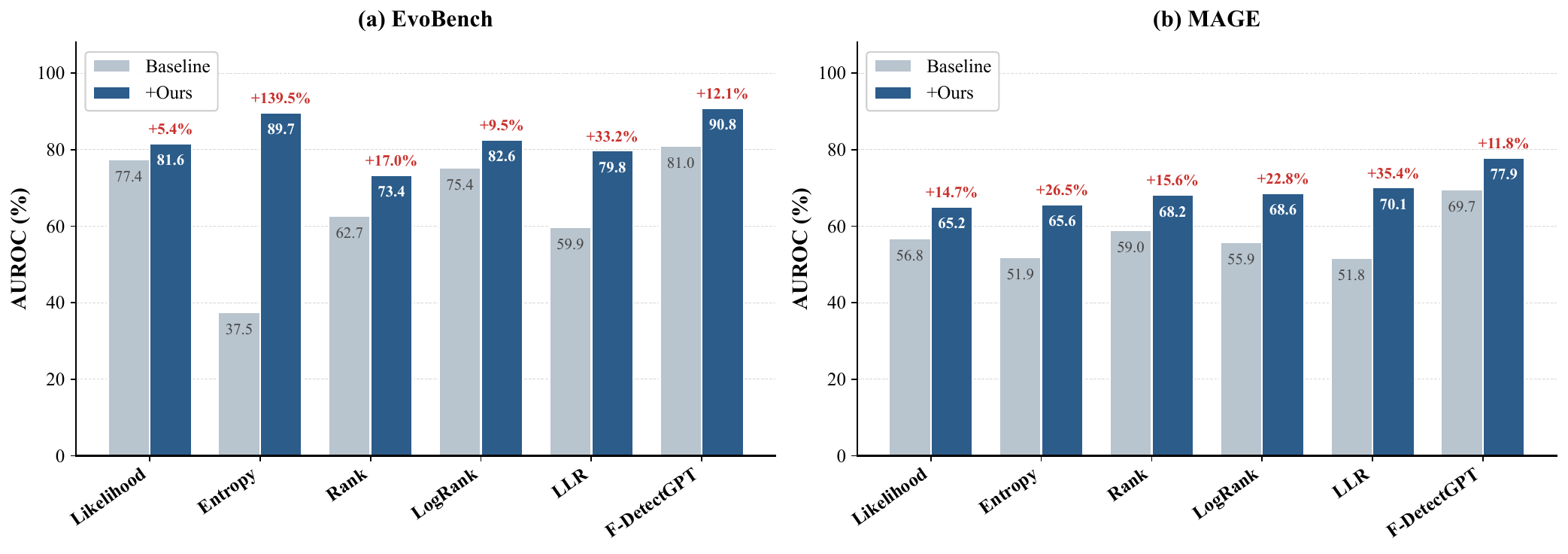}
    \caption{Performance improvement from applying DetectRouter to each zero-shot detection method. Bars show average AUROC on EvoBench and MAGE; numbers indicate relative improvement.}
    \label{fig:routing_improvement}
\end{figure*}
\paragraph{Training Data.}
We train DetectRouter on the MIRAGE benchmark~\citep{fu2025detectanyllm}, a comprehensive multi-task dataset spanning five domains---Academic, News, Comment, Email, and Website---and three generation tasks: \textit{Generate}, \textit{Polish}, and \textit{Rewrite}. For Stage~1, we synthesize 30K training samples by prompting locally cached surrogate models to generate text aligned with the MIRAGE schema. For Stage~2, we leverage MIRAGE directly with detection scores annotated across our full detector pool, providing supervision for learning text-detector affinity patterns.

\paragraph{Surrogate Models and Detection Methods.}
Our surrogate pool comprises ten open-source LLMs spanning 1.5B to 20B parameters: GPT-J-6B~\citep{wang2021gptj}, GPT-NeoX-20B~\citep{biderman2023gptneox20b}, GPT-Neo-2.7B~\citep{black2021gptneo}, GPT2-XL~\citep{radford2019language}, GPT-oss-20B~\citep{openai2025gptoss20b}, LLaMA3-8B~\citep{dubey2024llama3}, Falcon-7B~\citep{almazrouei2023falcon}, DeepSeek-7B~\citep{deepseek2023llm}, OPT-2.7B~\citep{zhang2022opt}, and Qwen-3B~\citep{bai2023qwen}. This selection ensures diversity across architectures and training corpora. We evaluate six zero-shot detection criteria---Likelihood, Entropy, Rank, LogRank, LLR, and Fast-DetectGPT~\citep{bao2024fastdetectgpt}---comparing each fixed-surrogate baseline against our routed variant. We first benchmark all detection methods across all ten surrogate models and identify LLaMA3-8B as the consistently best-performing surrogate; we therefore fix all baseline methods to use LLaMA3-8B, ensuring the most competitive comparison rather than an artificially weakened baseline.

\paragraph{Router Configuration.}
The router employs LLM2Vec-LLaMA3-8B~\citep{behnamghader2024llm2vec} as the encoder for main experiments, with lightweight alternatives explored in ablation studies. Each class is represented by 10 learnable prototypes with cosine similarity and temperature $\tau{=}10$. Both training stages run for 8 epochs with batch size 32; Stage~1 uses learning rate $2{\times}10^{-5}$ while Stage~2 reduces to $10^{-5}$ and balances cross-entropy, distance regression, and prototype anchoring losses.

\paragraph{Baselines.}
We compare against representative zero-shot detection methods including token-level statistics such as Likelihood~\citep{solaiman2019release}, Entropy and Rank~\citep{gehrmann2019gltr}, and LogRank~\citep{mitchell2023detectgpt}; relative scoring via LLR~\citep{su2023detectllm}; and perturbation-based Fast-DetectGPT~\citep{bao2024fastdetectgpt}. We also include recent approaches: Lastde~\citep{xu2024training}, FourierGPT~\citep{xu2024detecting}, Diveye~\citep{basani2025diversity}, and UCE~\citep{hou2025theoretical}. To isolate the contribution of our routing framework from the training data, we additionally train supervised classifiers on the same Stage~2 data. All baselines use the same surrogate model for fair comparison.

\begin{table*}[t]
\centering
\caption{Detection performance on EvoBench and MAGE. Fixed-surrogate baselines use LLaMA3-8B. AUROC is reported in percentage.}
\label{tab:main_results}
\resizebox{0.95\textwidth}{!}{
\begin{tabular}{l|ccccccc|c|c}
\toprule
& \multicolumn{8}{c|}{EvoBench} & MAGE \\
\cmidrule{2-9} \cmidrule{10-10}
Method & LLaMA3 & LLaMA2 & Claude & GPT-4o & Gemini & GPT-4 & Qwen & Avg & Avg \\
\midrule
Likelihood & 88.84 & 61.68 & 79.70 & 77.80 & 78.89 & 77.80 & 77.40 & 77.44 & 56.81 \\
Entropy & 28.85 & 57.58 & 33.61 & 35.43 & 38.34 & 38.10 & 30.34 & 37.46 & 51.90 \\
Rank & 68.49 & 58.38 & 64.50 & 64.64 & 64.46 & 63.43 & 55.06 & 62.71 & 58.99 \\
LogRank & 86.47 & 59.99 & 77.37 & 76.80 & 75.99 & 74.87 & 76.10 & 75.37 & 55.91 \\
LLR & 66.46 & 49.61 & 61.73 & 60.57 & 59.12 & 58.43 & 63.08 & 59.86 & 51.78 \\
Fast-DetectGPT & 94.17 & 80.27 & 81.57 & 74.89 & 80.48 & 75.58 & 80.10 & 81.01 & 69.69 \\
Lastde & 90.80 & 80.54 & 76.75 & 71.34 & 76.50 & 71.51 & 69.20 & 76.66 & 70.71 \\
FourierGPT & 63.99 & 57.23 & 58.20 & 55.52 & 56.69 & 55.16 & 64.39 & 58.74 & 60.34 \\
Diveye & 84.65 & 73.46 & 74.21 & 73.32 & 74.06 & 72.28 & 68.85 & 74.40 & 69.60 \\
UCE & 75.62 & 61.35 & 68.72 & 62.04 & 61.28 & 59.38 & 63.63 & 64.57 & 54.17 \\
Supervised & 74.96 & 66.88 & 73.38 & 77.66 & 62.07 & 76.16 & 66.49 & 71.26 & 73.92 \\
\rowcolor{gray!20}
DetectRouter & \textbf{95.66} & \textbf{80.82} & \textbf{94.22} & \textbf{94.29} & \textbf{93.79} & \textbf{94.16} & \textbf{83.01} & \textbf{90.85} & \textbf{77.92} \\
\bottomrule
\end{tabular}
}
\end{table*}

\paragraph{Evaluation Benchmarks.}
We evaluate on two complementary in-the-wild benchmarks. MAGE~\citep{li2024mage} provides broad coverage with 30,265 test pairs across eight model families---OpenAI, OPT, FLAN, LLaMA, BigScience, EleutherAI, GLM, and human paraphrase---representing established generators and diverse domains. EvoBench~\citep{yu2025evobench} specifically targets the evolving nature of LLMs, comprising 32,993 pairs across seven families and 29 evolving versions, capturing both publisher updates and developer modifications such as fine-tuning and pruning. Following prior work~\citep{mitchell2023detectgpt,bao2024fastdetectgpt}, we report AUROC as the primary metric. Complete details on data construction, hyperparameters, and benchmark statistics are provided in Appendix~\ref{app:exp_setup}.
\subsection{Main Results}
A central contribution of DetectRouter is its ability to serve as a universal enhancement layer for any zero-shot detection criterion. We first demonstrate the universal enhancement property before presenting state-of-the-art comparisons.

\begin{table}[t]
\centering
\caption{Ablation on training stages and loss components. Average AUROC is reported in percentage. $\Delta$ denotes relative drop from the full method.}
\label{tab:ablation_components}
\small
\begin{tabular}{l|cc|c}
\toprule
Setting & EvoBench & MAGE & $\Delta$ \\
\midrule
Fixed Surrogate & 81.01 & 69.69 & -10.7\% \\
\midrule
Direct Classification & 83.08 & 75.34 & -5.9\% \\
Stage 1 Only & 72.87 & 75.67 & -11.3\% \\
Stage 2 Only & 87.39 & 74.49 & -4.1\% \\
\midrule
Full w/o $\mathcal{L}_{KL}$ & 87.07 & 75.56 & -3.6\% \\
Full w/o $\mathcal{L}_{sep}$ & 88.53 & 74.41 & -3.5\% \\
Full w/o $\mathcal{L}_{anchor}$ & 89.56 & 74.90 & -2.6\% \\
\midrule
Full & 90.85 & 77.92 & -- \\
\bottomrule
\end{tabular}
\end{table}

\paragraph{Routing as a Universal Enhancement.}
Figure~\ref{fig:routing_improvement} visualizes the performance gains achieved by applying our routing mechanism to each of the six detection methods. On EvoBench, routing provides relative improvements ranging from 5.4\% for Likelihood to 139.4\% for Entropy. On MAGE, improvements range from 11.8\% for Fast-DetectGPT to 35.4\% for LLR. Entropy-based detection benefits most dramatically: the fixed-surrogate baseline achieves only 37.46\% on EvoBench, falling below random chance for several model families, while routing elevates performance to 89.70\%. This demonstrates that our framework is not tied to any specific detection criterion but serves as a general-purpose enhancement layer that can be seamlessly integrated with existing methods. Per-family breakdown for all method combinations is provided in Appendix~\ref{app:routing_details}.

\begin{table}[t]
\centering
\caption{Ablation on router encoder architectures. Average AUROC is reported in percentage alongside model parameters and inference latency per sample.}
\label{tab:ablation_encoder}
\resizebox{\columnwidth}{!}{
\begin{tabular}{l|cc|cc}
\toprule
Encoder & EvoBench & MAGE & Params & Latency \\
\midrule
Fixed Surrogate & 82.53 & 69.68 & -- & -- \\
\midrule
BERT-base & 86.41 & 71.07 & 109M & 4.87ms \\
RoBERTa-base & 88.34 & 74.10 & 125M & 4.87ms \\
RoBERTa-large & 88.18 & 75.59 & 355M & 12.70ms \\
\midrule
LLM2Vec-LLaMA3 & 90.85 & 77.92 & 8B & 166.2ms \\
\bottomrule
\end{tabular}
}
\end{table}

\paragraph{Comparison with Baselines.}
Table~\ref{tab:main_results} presents AUROC scores on EvoBench with per-family breakdown and MAGE average. DetectRouter achieves the best overall performance at 90.85\% on EvoBench and 77.92\% on MAGE, substantially outperforming all baselines. Among fixed-surrogate methods, Fast-DetectGPT is the strongest at 81.01\% on EvoBench and 69.69\% on MAGE; our routing mechanism improves upon this baseline by 9.84 and 8.23 percentage points respectively. The improvement is particularly pronounced on proprietary models where distributional gaps are largest: on Claude, performance increases from 81.57\% to 94.22\%; on GPT-4o, from 74.89\% to 94.29\%; and on GPT-4, from 75.58\% to 94.16\%. Among recent zero-shot methods, Lastde achieves the highest average at 76.66\% on EvoBench, yet DetectRouter outperforms it by 14.19 percentage points. FourierGPT, Diveye, and UCE all fall below our method by substantial margins, demonstrating that adaptive routing provides orthogonal benefits to detection criterion design. Full per-family results on MAGE are provided in Appendix~\ref{app:mage_full}.

To isolate the contribution of our routing framework from the training data, we train supervised classifiers on the same Stage~2 data. The best supervised baseline using LLama3-8b representations achieves 71.26\% on EvoBench and 73.92\% on MAGE. DetectRouter outperforms this baseline by 12.93\% points on EvoBench while matching performance on MAGE, demonstrating that the gains stem from learned text-detector affinity patterns rather than simply utilizing more training data. Full results for all ten surrogate configurations are provided in Appendix~\ref{app:supervised}.

\subsection{Ablation and Analysis}


\begin{figure}[t]
    \centering
    \includegraphics[width=1.0\linewidth]{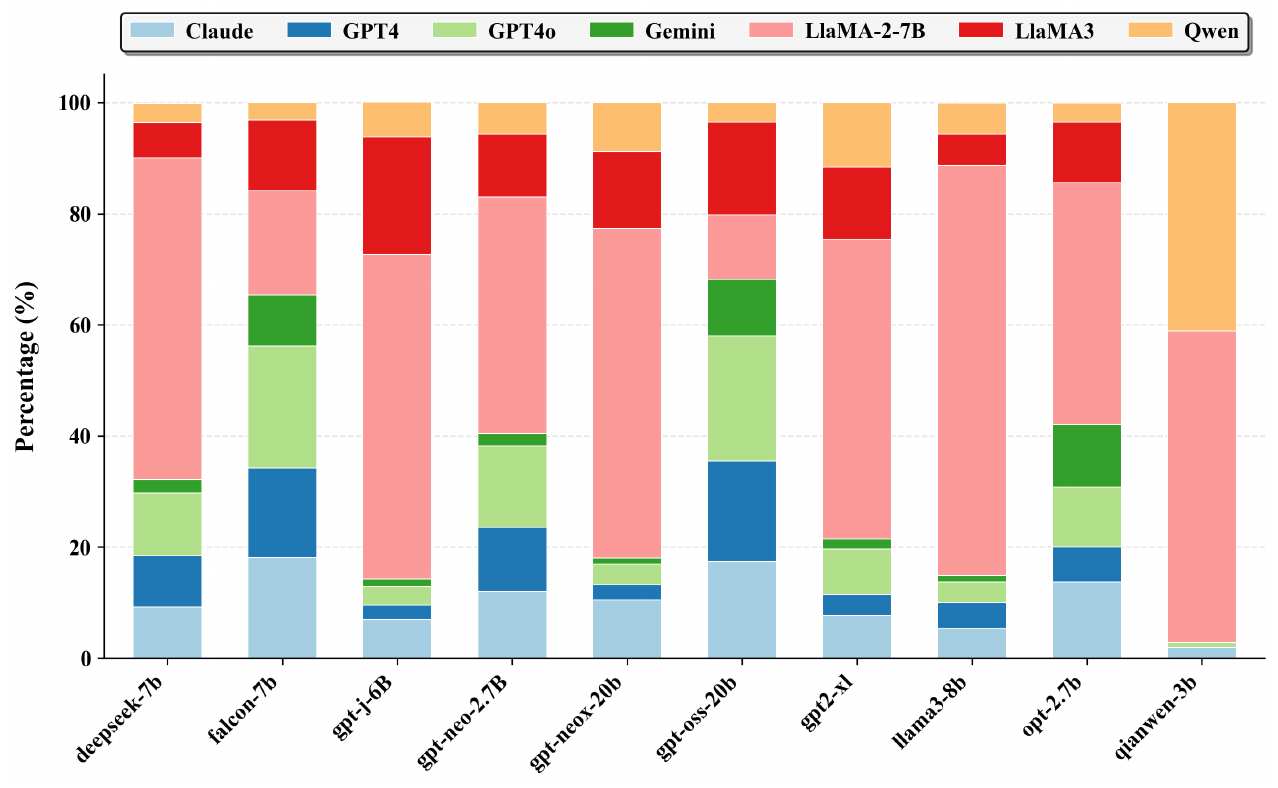}
    \caption{Distribution of selected surrogates for each source model family, confirming source-specific routing preferences.}
    \label{fig:analysis_routing}
\end{figure}

\paragraph{Training Stages and Loss Components.}
Table~\ref{tab:ablation_components} presents ablation results. Both training stages prove essential: Stage 1 Only suffers $-$11.3\% degradation as prototypes trained solely on white-box models fail to generalize, while Stage 2 Only shows a $-$4.1\% gap from lacking well-initialized prototypes. Each loss component contributes meaningfully---replacing KL divergence with $\mathcal{L}_1$ regression causes $-$3.6\% drop, and removing $\mathcal{L}_{sep}$ results in $-$3.5\% degradation. Direct classification bypasses the prototype framework and exhibits $-$5.9\% gap, confirming the advantage of metric-space routing.

\paragraph{Encoder Architecture.}
Table~\ref{tab:ablation_encoder} compares different encoder architectures for the routing network. Among lightweight encoders, RoBERTa-base achieves the best balance, outperforming BERT-base by 1.93\% on EvoBench and 3.03\% on MAGE while maintaining identical inference latency at 4.87ms per sample. LLM2Vec-LLaMA3 achieves the best overall performance at 90.85\% on EvoBench and 77.92\% on MAGE, but requires 166.2ms per sample. This trade-off suggests that for latency-sensitive applications, RoBERTa-base offers an attractive operating point with strong performance at minimal computational overhead.

\paragraph{Routing Pattern Interpretation.}
To validate that our performance gains stem from adaptive alignment rather than simply ensembling, we visualize the distribution of selected surrogates for different source families in Figure~\ref{fig:analysis_routing}. The stacked bar chart reveals that the router does not simply default to a single universal detector for all inputs. Instead, it exhibits distinct, source-specific preferences. For instance, when detecting text from DeepSeek-7B, the router frequently selects structurally similar surrogates, whereas for Falcon-7B, the preference shifts significantly toward a different subset of detectors. This confirms our hypothesis that affinity is a latent, continuous property: the router successfully identifies and exploits the unique statistical signatures of each source.

\begin{figure}[t]
    \centering
    \includegraphics[width=1.0\linewidth]{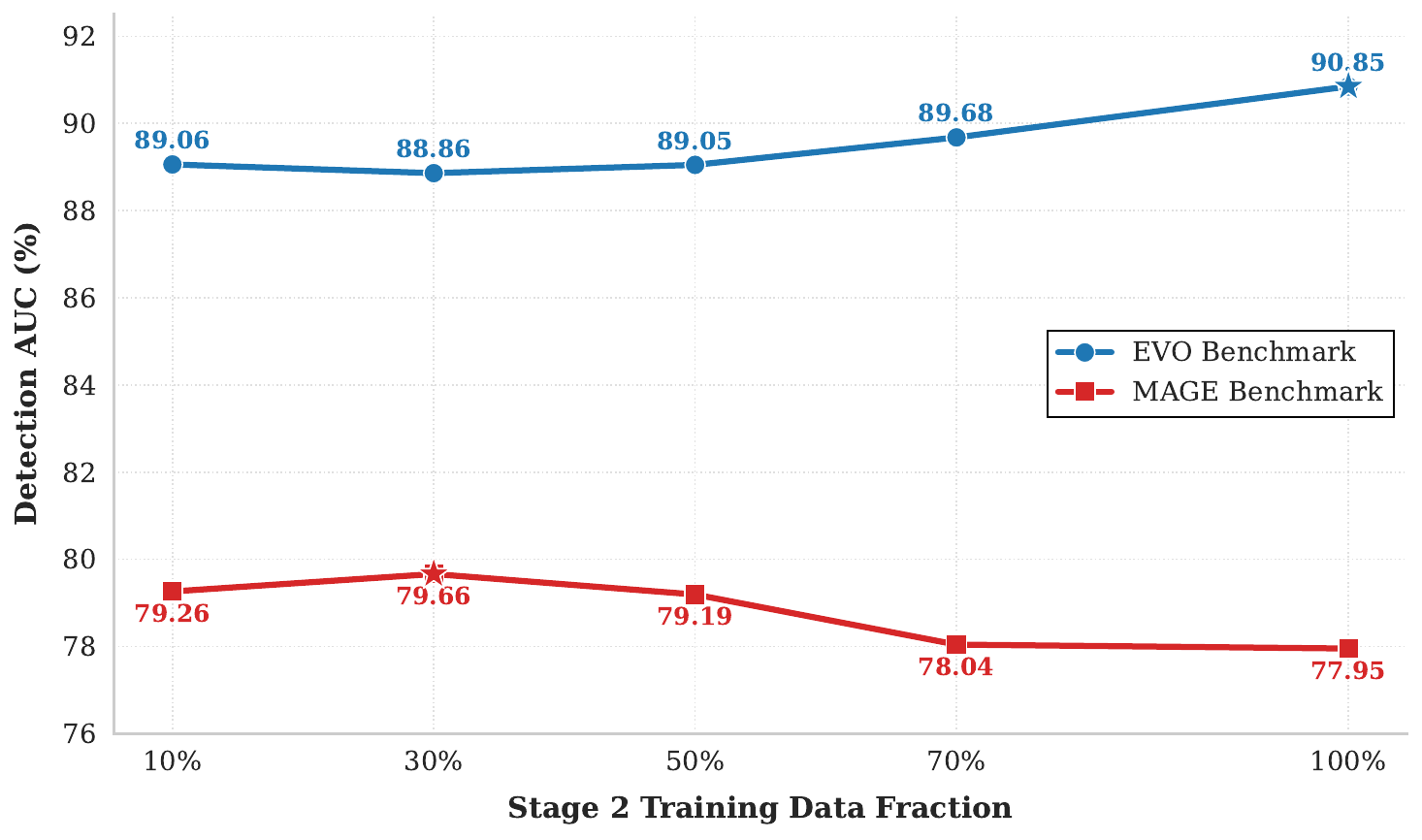}
    \caption{Detection AUC scores on EvoBench and MAGE plotted against the fraction of training data used in Stage~2.}
    \label{fig:data_efficiency}
\end{figure}

\paragraph{Data Efficiency.}
A key practical concern for black-box adaptation is the cost of collecting detection scores. We investigate the data efficiency of DetectRouter by varying the fraction of Stage~2 training data from 10\% to 100\%. As shown in Figure~\ref{fig:data_efficiency}, the model exhibits remarkable stability. On EvoBench, the router achieves an AUC of 89.06\% with only 10\% of the data, within 1.8 percentage points of the performance achieved with the full dataset. This indicates that the prototype space constructed in Stage~1 provides a strong regularization prior, making DetectRouter highly cost-effective for adapting to new models. Additional analysis on prototype space geometry is provided in Appendix~\ref{app:analysis}.

\section{Conclusion}

This work reveals that zero-shot LLM detection is highly sensitive to the distributional alignment between the surrogate model and the unknown source. Through systematic experiments and theoretical analysis, we demonstrate that surrogate-source mismatch leads to substantial performance degradation, and derive a bound showing the gap is governed by KL divergence between distributions. Building on this finding, we propose DetectRouter, the first framework to leverage multiple surrogate models as complementary experts and dynamically route each input to its most compatible detector. Experiments confirm that this adaptive routing strategy yields consistent improvements across all detection criteria and achieves state-of-the-art performance on two benchmarks.

\section*{Impact Statement}

This work aims to advance the reliable detection of machine-generated text, which has broad implications for maintaining information integrity in the age of large language models. Effective detection methods can help combat the spread of AI-generated disinformation, preserve academic integrity, and prevent the contamination of training corpora with synthetic content. Our routing framework improves detection robustness across diverse and evolving LLMs, making it more practical for real-world deployment where the source model is unknown.

We acknowledge that detection technologies exist in an ongoing adversarial dynamic with generation methods. While our work strengthens the defensive capability, determined adversaries may develop countermeasures. We also note that imperfect detectors risk false accusations against human authors, which could have serious consequences in educational or professional contexts. We encourage practitioners to use detection tools as one signal among many rather than as definitive judgments, and to establish appropriate thresholds that balance false positive and false negative rates for their specific applications.

\bibliography{main}
\bibliographystyle{icml2026}

\newpage
\appendix
\onecolumn

\appendix

\section{Theoretical Analysis}
\label{app:theory_proof}

This section provides the complete proof of Proposition~\ref{prop:mismatch_bound}, establishing that the performance gap in zero-shot detection is fundamentally bounded by the distributional divergence between source and surrogate models.

\subsection{Preliminaries}

We introduce the necessary background on probability divergences.

\begin{definition}[Total Variation Distance]
For two probability distributions $P$ and $Q$ defined on a measurable space $(\mathcal{X}, \mathcal{F})$, the total variation distance is
\begin{equation}
    d_{TV}(P, Q) = \sup_{A \in \mathcal{F}} |P(A) - Q(A)| = \frac{1}{2} \int_{\mathcal{X}} |p(x) - q(x)| \, dx
\end{equation}
where $p$ and $q$ denote the probability density functions of $P$ and $Q$, respectively.
\end{definition}

\begin{definition}[Kullback-Leibler Divergence]
For two probability distributions $P$ and $Q$ with $P$ absolutely continuous with respect to $Q$, the KL divergence is
\begin{equation}
    D_{KL}(P \| Q) = \int_{\mathcal{X}} p(x) \log \frac{p(x)}{q(x)} \, dx
\end{equation}
\end{definition}

\begin{lemma}[Pinsker's Inequality]
\label{lemma:pinsker}
For any two probability distributions $P$ and $Q$:
\begin{equation}
    d_{TV}(P, Q) \leq \sqrt{\frac{1}{2} D_{KL}(P \| Q)}
\end{equation}
\end{lemma}

This classical result connects the total variation distance, which is directly useful for bounding expectation differences, to the KL divergence, which has information-theoretic interpretations and is often more tractable to estimate.

\subsection{Proof of the Mismatch Risk Bound}

\begin{proof}
Let $T(x): \mathcal{X} \to \mathbb{R}$ be a zero-shot scoring function with $|T(x)| \leq B$ for all $x \in \mathcal{X}$. Let $P_{src}$ denote the source distribution and $P_{sur}$ denote the surrogate distribution, with corresponding density functions $p_{src}$ and $p_{sur}$.

The optimal detection signal is
\begin{equation}
    \mu^* = \mathbb{E}_{x \sim P_{src}}[T(x)] = \int_{\mathcal{X}} T(x) p_{src}(x) \, dx
\end{equation}

The proxy signal obtained using the surrogate is
\begin{equation}
    \mu_{proxy} = \mathbb{E}_{x \sim P_{sur}}[T(x)] = \int_{\mathcal{X}} T(x) p_{sur}(x) \, dx
\end{equation}

\textit{Step 1: Express the gap as an integral.} The performance gap can be written as
\begin{align}
    |\mu^* - \mu_{proxy}| &= \left| \int_{\mathcal{X}} T(x) p_{src}(x) \, dx - \int_{\mathcal{X}} T(x) p_{sur}(x) \, dx \right| \\
    &= \left| \int_{\mathcal{X}} T(x) \left( p_{src}(x) - p_{sur}(x) \right) dx \right|
\end{align}

\textit{Step 2: Apply the boundedness condition.} Since $|T(x)| \leq B$, we have
\begin{align}
    \left| \int_{\mathcal{X}} T(x) \left( p_{src}(x) - p_{sur}(x) \right) dx \right| &\leq \int_{\mathcal{X}} |T(x)| \cdot |p_{src}(x) - p_{sur}(x)| \, dx \\
    &\leq B \int_{\mathcal{X}} |p_{src}(x) - p_{sur}(x)| \, dx
\end{align}

\textit{Step 3: Relate to total variation distance.} By definition:
\begin{equation}
    \int_{\mathcal{X}} |p_{src}(x) - p_{sur}(x)| \, dx = 2 \cdot d_{TV}(P_{src}, P_{sur})
\end{equation}

Therefore:
\begin{equation}
    |\mu^* - \mu_{proxy}| \leq 2B \cdot d_{TV}(P_{src}, P_{sur})
\end{equation}

\textit{Step 4: Apply Pinsker's inequality.} Using Lemma~\ref{lemma:pinsker}:
\begin{align}
    |\mu^* - \mu_{proxy}| &\leq 2B \cdot d_{TV}(P_{src}, P_{sur}) \\
    &\leq 2B \cdot \sqrt{\frac{1}{2} D_{KL}(P_{src} \| P_{sur})} \\
    &= B \cdot \sqrt{2 D_{KL}(P_{src} \| P_{sur})}
\end{align}

This completes the proof.
\end{proof}

\subsection{Implications for Zero-Shot Detection}

The Mismatch Risk Bound has several important implications for zero-shot LLM detection.

\paragraph{Monotonicity in Divergence.}
The bound is monotonically increasing in $D_{KL}(P_{src} \| P_{sur})$. As the distributional divergence between source and surrogate increases, the potential performance degradation also increases. This theoretical prediction aligns precisely with our empirical observations in Table~\ref{tab:analysis_cost}, where cross-model performance degrades significantly compared to the matched setting.

\paragraph{Optimality of Routing.}
Given a pool of candidate surrogates $\{M_1, M_2, \ldots, M_K\}$ with corresponding distributions $\{P_{M_1}, P_{M_2}, \ldots, P_{M_K}\}$, the bound suggests that the optimal surrogate for a source $P_{src}$ is
\begin{equation}
    M^* = \arg\min_{k \in [K]} D_{KL}(P_{src} \| P_{M_k})
\end{equation}
This minimization tightens the error bound and recovers the best possible detection performance within the surrogate pool, directly motivating our routing approach.

\paragraph{Architectural Affinity.}
The structured patterns observed in Figure~\ref{fig:analysis_confusion}, where models from the same family exhibit higher mutual affinity, can be explained through this framework. Models sharing architectural design choices and pre-training data are likely to have smaller KL divergence, leading to tighter bounds and better cross-detection performance.

\paragraph{Variance in Black-box Settings.}
The extreme performance variance observed in Figure~\ref{fig:intro} corresponds to the varying values of $D_{KL}(P_{src} \| P_{M_k})$ across different surrogates $M_k$. For a fixed black-box source, some surrogates happen to have distributions closer to the source, resulting in high detection AUC, while others have large divergence, leading to near-random performance.

\subsection{Extension to Finite Samples}

In practice, we work with finite samples rather than true distributions. The bound can be extended to account for estimation error using standard concentration inequalities.

\begin{corollary}[Finite Sample Bound]
Let $\hat{\mu}^*$ and $\hat{\mu}_{proxy}$ be empirical estimates of $\mu^*$ and $\mu_{proxy}$ based on $n$ i.i.d.\ samples. Then with probability at least $1 - \delta$:
\begin{equation}
    |\hat{\mu}^* - \hat{\mu}_{proxy}| \leq B \cdot \sqrt{2 D_{KL}(P_{src} \| P_{sur})} + 2B\sqrt{\frac{\log(2/\delta)}{2n}}
\end{equation}
\end{corollary}

The additional term captures the statistical uncertainty from finite sampling, which vanishes as $n \to \infty$.

\subsection{Connection to Information Theory}

The appearance of KL divergence in the bound is not coincidental. From an information-theoretic perspective, $D_{KL}(P_{src} \| P_{sur})$ quantifies the information lost when using $P_{sur}$ to approximate $P_{src}$. In the context of LLM detection, this information loss directly translates to degraded detection performance, as the surrogate model fails to capture the distinctive statistical patterns of the source.

This connection suggests that effective routing can be viewed as an information-theoretic channel matching problem: we seek to route each input to the surrogate that minimizes the information loss, thereby maximizing the fidelity of the detection signal.

\section{Additional Experimental Results}
\label{app:exp_details}

This section provides comprehensive experimental results that complement the main paper.

\subsection{Full MAGE Benchmark Results}
\label{app:mage_full}

Table~\ref{tab:mage_full} presents the complete per-family breakdown on the MAGE benchmark. The results reveal substantial heterogeneity in baseline performance across model families. For instance, Likelihood achieves 92.96\% on LLaMA-family text but only 0.26\% on EleutherAI, reflecting severe surrogate-source mismatch when architectures differ substantially. Our routing mechanism effectively mitigates this variance by dynamically selecting appropriate surrogates.

\begin{table*}[t]
\centering
\caption{Detection performance on MAGE across eight model families. AUROC is reported in percentage. The best result in each column is highlighted in bold.}
\label{tab:mage_full}
\resizebox{\textwidth}{!}{
\begin{tabular}{l|cccccccc|c}
\toprule
Method & OpenAI & GLM & Human-Para & FLAN & LLaMA & BigScience & OPT & EleutherAI & Avg \\
\midrule
Likelihood & 89.12 & 91.30 & 60.61 & 30.98 & \textbf{92.96} & 37.51 & 31.21 & 0.26 & 56.81 \\
Entropy & 23.08 & 17.98 & 62.83 & 66.65 & 20.30 & 66.22 & 76.39 & 81.76 & 51.90 \\
Rank & 73.74 & 83.12 & 53.58 & 34.93 & 81.74 & 45.49 & 52.75 & 46.63 & 58.99 \\
LogRank & 87.05 & 90.28 & 53.81 & 31.56 & 91.69 & 37.38 & 32.39 & 23.10 & 55.91 \\
LLR & 72.90 & 74.24 & 36.21 & 36.17 & 78.09 & 39.97 & 40.17 & 36.52 & 51.78 \\
Fast-DetectGPT & 76.40 & 80.41 & 61.92 & \textbf{76.29} & 67.35 & 72.40 & 64.61 & 58.12 & 69.69 \\
\midrule
Lastde & 74.47 & 83.87 & 62.96 & 66.35 & 72.01 & 64.63 & 72.35 & 69.07 & 70.71 \\
FourierGPT & 59.31 & 84.88 & 68.32 & 4.72 & 84.39 & 39.60 & 57.55 & 63.97 & 60.34 \\
Diveye & 87.89 & \textbf{92.61} & 62.10 & 34.65 & 93.80 & 52.00 & 67.60 & 66.11 & 69.60 \\
UCE & 18.81 & 20.02 & 70.60 & 67.05 & 23.48 & 62.86 & \textbf{82.81} & \textbf{87.73} & 54.17 \\
\midrule
\rowcolor{gray!20}
DetectRouter & \textbf{93.24} & 90.37 & \textbf{88.52} & 76.66 & 72.55 & \textbf{74.65} & 67.56 & 59.81 & \textbf{77.92} \\
\bottomrule
\end{tabular}
}
\end{table*}

\subsection{Full Routing Enhancement Results}
\label{app:routing_details}

Tables~\ref{tab:routing_evo_full} and~\ref{tab:routing_mage_full} present the complete per-family results for applying DetectRouter to each zero-shot detection method on EvoBench and MAGE, respectively.

\begin{table*}[t]
\centering
\caption{Routing enhancement on EvoBench. Per-family AUROC is reported in percentage for each detection method with and without routing. $\Delta$ denotes relative improvement from applying DetectRouter.}
\label{tab:routing_evo_full}
\resizebox{\textwidth}{!}{
\begin{tabular}{l|ccccccc|c|c}
\toprule
Method & LLaMA3 & LLaMA2 & Claude & GPT-4o & Gemini & GPT-4 & Qwen & Avg & $\Delta$ \\
\midrule
Likelihood & 88.84 & 61.68 & 79.70 & 77.80 & 78.89 & 77.80 & 77.40 & 77.44 & \multirow{2}{*}{+5.4\%} \\
\cellcolor{gray!20}Likelihood+Ours & \cellcolor{gray!20}90.58 & \cellcolor{gray!20}66.12 & \cellcolor{gray!20}84.56 & \cellcolor{gray!20}85.92 & \cellcolor{gray!20}84.20 & \cellcolor{gray!20}84.44 & \cellcolor{gray!20}75.31 & \cellcolor{gray!20}81.59 & \\
\midrule
Entropy & 28.85 & 57.58 & 33.61 & 35.43 & 38.34 & 38.10 & 30.34 & 37.46 & \multirow{2}{*}{+139.4\%} \\
\cellcolor{gray!20}Entropy+Ours & \cellcolor{gray!20}94.78 & \cellcolor{gray!20}80.20 & \cellcolor{gray!20}92.75 & \cellcolor{gray!20}92.24 & \cellcolor{gray!20}92.36 & \cellcolor{gray!20}92.59 & \cellcolor{gray!20}83.01 & \cellcolor{gray!20}89.70 & \\
\midrule
Rank & 68.49 & 58.38 & 64.50 & 64.64 & 64.46 & 63.43 & 55.06 & 62.71 & \multirow{2}{*}{+17.0\%} \\
\cellcolor{gray!20}Rank+Ours & \cellcolor{gray!20}78.13 & \cellcolor{gray!20}69.24 & \cellcolor{gray!20}75.01 & \cellcolor{gray!20}76.35 & \cellcolor{gray!20}74.34 & \cellcolor{gray!20}75.06 & \cellcolor{gray!20}65.48 & \cellcolor{gray!20}73.37 & \\
\midrule
LogRank & 86.47 & 59.99 & 77.37 & 76.80 & 75.99 & 74.87 & 76.10 & 75.37 & \multirow{2}{*}{+9.5\%} \\
\cellcolor{gray!20}LogRank+Ours & \cellcolor{gray!20}90.26 & \cellcolor{gray!20}70.31 & \cellcolor{gray!20}83.82 & \cellcolor{gray!20}86.07 & \cellcolor{gray!20}84.35 & \cellcolor{gray!20}84.47 & \cellcolor{gray!20}78.67 & \cellcolor{gray!20}82.56 & \\
\midrule
LLR & 66.46 & 49.61 & 61.73 & 60.57 & 59.12 & 58.43 & 63.08 & 59.86 & \multirow{2}{*}{+33.2\%} \\
\cellcolor{gray!20}LLR+Ours & \cellcolor{gray!20}85.97 & \cellcolor{gray!20}71.96 & \cellcolor{gray!20}80.05 & \cellcolor{gray!20}81.57 & \cellcolor{gray!20}81.46 & \cellcolor{gray!20}80.66 & \cellcolor{gray!20}76.68 & \cellcolor{gray!20}79.76 & \\
\midrule
Fast-DetectGPT & 94.17 & 80.27 & 81.57 & 74.89 & 80.48 & 75.58 & 80.10 & 81.01 & \multirow{2}{*}{+12.1\%} \\
\cellcolor{gray!20}Fast-DetectGPT+Ours & \cellcolor{gray!20}95.66 & \cellcolor{gray!20}80.82 & \cellcolor{gray!20}94.22 & \cellcolor{gray!20}94.29 & \cellcolor{gray!20}93.79 & \cellcolor{gray!20}94.16 & \cellcolor{gray!20}83.01 & \cellcolor{gray!20}90.85 & \\
\bottomrule
\end{tabular}
}
\end{table*}

\begin{table*}[t]
\centering
\caption{Routing enhancement on MAGE. Per-family AUROC is reported in percentage for each detection method with and without routing. $\Delta$ denotes relative improvement from applying DetectRouter.}
\label{tab:routing_mage_full}
\resizebox{\textwidth}{!}{
\begin{tabular}{l|cccccccc|c|c}
\toprule
Method & OpenAI & GLM & Human-Para & FLAN & LLaMA & BigScience & OPT & EleutherAI & Avg & $\Delta$ \\
\midrule
Likelihood & 89.12 & 91.30 & 60.61 & 30.98 & 92.96 & 37.51 & 31.21 & 0.26 & 56.81 & \multirow{2}{*}{+14.7\%} \\
\cellcolor{gray!20}Likelihood+Ours & \cellcolor{gray!20}91.67 & \cellcolor{gray!20}85.68 & \cellcolor{gray!20}83.46 & \cellcolor{gray!20}54.41 & \cellcolor{gray!20}72.62 & \cellcolor{gray!20}55.02 & \cellcolor{gray!20}47.54 & \cellcolor{gray!20}30.85 & \cellcolor{gray!20}65.16 & \\
\midrule
Entropy & 23.08 & 17.98 & 62.83 & 66.65 & 20.30 & 66.22 & 76.39 & 81.76 & 51.90 & \multirow{2}{*}{+26.5\%} \\
\cellcolor{gray!20}Entropy+Ours & \cellcolor{gray!20}76.16 & \cellcolor{gray!20}42.35 & \cellcolor{gray!20}86.80 & \cellcolor{gray!20}68.37 & \cellcolor{gray!20}33.29 & \cellcolor{gray!20}68.62 & \cellcolor{gray!20}73.12 & \cellcolor{gray!20}76.37 & \cellcolor{gray!20}65.64 & \\
\midrule
Rank & 73.74 & 83.12 & 53.58 & 34.93 & 81.74 & 45.49 & 52.75 & 46.63 & 58.99 & \multirow{2}{*}{+15.6\%} \\
\cellcolor{gray!20}Rank+Ours & \cellcolor{gray!20}83.54 & \cellcolor{gray!20}85.74 & \cellcolor{gray!20}79.94 & \cellcolor{gray!20}53.48 & \cellcolor{gray!20}74.40 & \cellcolor{gray!20}58.91 & \cellcolor{gray!20}61.61 & \cellcolor{gray!20}47.90 & \cellcolor{gray!20}68.19 & \\
\midrule
LogRank & 87.05 & 90.28 & 53.81 & 31.56 & 91.69 & 37.38 & 32.39 & 23.10 & 55.91 & \multirow{2}{*}{+22.8\%} \\
\cellcolor{gray!20}LogRank+Ours & \cellcolor{gray!20}91.94 & \cellcolor{gray!20}90.39 & \cellcolor{gray!20}80.81 & \cellcolor{gray!20}54.65 & \cellcolor{gray!20}79.04 & \cellcolor{gray!20}58.98 & \cellcolor{gray!20}54.19 & \cellcolor{gray!20}39.07 & \cellcolor{gray!20}68.63 & \\
\midrule
LLR & 72.90 & 74.24 & 36.21 & 36.17 & 78.09 & 39.97 & 40.17 & 36.52 & 51.78 & \multirow{2}{*}{+35.4\%} \\
\cellcolor{gray!20}LLR+Ours & \cellcolor{gray!20}89.40 & \cellcolor{gray!20}90.17 & \cellcolor{gray!20}72.46 & \cellcolor{gray!20}59.63 & \cellcolor{gray!20}81.28 & \cellcolor{gray!20}61.35 & \cellcolor{gray!20}59.98 & \cellcolor{gray!20}46.77 & \cellcolor{gray!20}70.13 & \\
\midrule
Fast-DetectGPT & 76.40 & 80.41 & 61.92 & 76.29 & 67.35 & 72.40 & 64.61 & 58.12 & 69.69 & \multirow{2}{*}{+11.8\%} \\
\cellcolor{gray!20}Fast-DetectGPT+Ours & \cellcolor{gray!20}93.24 & \cellcolor{gray!20}90.37 & \cellcolor{gray!20}88.52 & \cellcolor{gray!20}76.66 & \cellcolor{gray!20}72.55 & \cellcolor{gray!20}74.65 & \cellcolor{gray!20}67.56 & \cellcolor{gray!20}59.81 & \cellcolor{gray!20}77.92 & \\
\bottomrule
\end{tabular}
}
\end{table*}

\subsection{Supervised Baseline Results}
\label{app:supervised}

We evaluate supervised detection using each surrogate model's representations trained on Stage~2 data for binary classification. Tables~\ref{tab:supervised_evo} and~\ref{tab:supervised_mage} present the complete results on EvoBench and MAGE, respectively. Performance varies substantially across surrogates, demonstrating that supervised approaches are highly sensitive to the choice of representation model. For instance, DeepSeek-7B achieves 79.30\% on EvoBench while Falcon-7B only reaches 48.27\%, highlighting the importance of surrogate selection even in supervised settings. DetectRouter consistently outperforms all single-surrogate supervised baselines.

\begin{table*}[t]
\centering
\caption{Supervised detection results on EvoBench using different surrogate representations trained on Stage~2 data. AUROC is reported in percentage.}
\label{tab:supervised_evo}
\resizebox{\textwidth}{!}{
\begin{tabular}{l|ccccccc|c}
\toprule
Surrogate & LLaMA3 & LLaMA2 & Claude & GPT-4o & Gemini & GPT-4 & Qwen & Avg \\
\midrule
LLaMA3-8B & 74.96 & 66.88 & 73.38 & 77.66 & 62.07 & 76.16 & 66.49 & 71.09 \\
DeepSeek-7B & 85.71 & 68.16 & 83.89 & 84.64 & 74.05 & 84.85 & 73.81 & 79.30 \\
Falcon-7B & 43.47 & 56.67 & 42.32 & 46.58 & 50.48 & 47.86 & 50.54 & 48.27 \\
GPT2-XL & 61.16 & 61.81 & 65.08 & 62.51 & 74.42 & 62.08 & 61.03 & 64.01 \\
GPT-J-6B & 73.92 & 61.72 & 72.85 & 73.74 & 54.55 & 72.66 & 68.78 & 68.32 \\
GPT-Neo-2.7B & 69.42 & 62.62 & 73.42 & 75.20 & 69.41 & 73.81 & 66.46 & 70.05 \\
OPT-2.7B & 72.79 & 64.32 & 68.75 & 72.12 & 68.17 & 72.16 & 65.64 & 69.14 \\
Qwen-3B & 82.22 & 64.96 & 81.18 & 80.59 & 76.22 & 80.81 & 69.56 & 76.51 \\
GPT-oss-20B & 50.07 & 57.80 & 47.90 & 51.15 & 50.72 & 51.93 & 60.73 & 52.90 \\
GPT-NeoX-20B & 57.93 & 56.44 & 58.39 & 57.56 & 57.37 & 57.74 & 55.89 & 57.33 \\
\midrule
\rowcolor{gray!20}
DetectRouter & \textbf{95.66} & \textbf{80.82} & \textbf{94.22} & \textbf{94.29} & \textbf{93.79} & \textbf{94.16} & \textbf{83.01} & \textbf{90.85} \\
\bottomrule
\end{tabular}
}
\end{table*}

\begin{table*}[t]
\centering
\caption{Supervised detection results on MAGE using different surrogate representations trained on Stage~2 data. AUROC is reported in percentage.}
\label{tab:supervised_mage}
\resizebox{\textwidth}{!}{
\begin{tabular}{l|cccccccc|c}
\toprule
Surrogate & OpenAI & GLM & Human-Para & FLAN & LLaMA & BigScience & OPT & EleutherAI & Avg \\
\midrule
LLaMA3-8B & 83.11 & 81.82 & 82.15 & 75.26 & 72.13 & 74.51 & 64.41 & 64.15 & 74.69 \\
DeepSeek-7B & 89.38 & 82.26 & 86.84 & 74.23 & 71.53 & 75.80 & 71.21 & 69.78 & 77.63 \\
Falcon-7B & 48.88 & 77.60 & 72.32 & 39.83 & 75.61 & 47.72 & 59.56 & 61.87 & 60.42 \\
GPT2-XL & 52.34 & 51.10 & 49.01 & 48.28 & 50.41 & 46.60 & 39.82 & 34.43 & 46.50 \\
GPT-J-6B & 81.61 & 77.60 & 78.14 & 63.61 & 70.17 & 62.57 & 50.07 & 49.31 & 66.64 \\
GPT-Neo-2.7B & 60.85 & 75.08 & 60.78 & 48.37 & 71.03 & 53.22 & 51.28 & 48.25 & 58.61 \\
OPT-2.7B & 78.13 & 74.96 & 74.66 & 67.43 & 67.32 & 62.51 & 52.10 & 48.35 & 65.68 \\
Qwen-3B & 88.66 & 81.14 & 80.99 & 68.58 & 68.20 & 65.34 & 62.09 & 56.71 & 71.46 \\
GPT-oss-20B & 53.76 & 61.59 & 63.87 & 44.68 & 59.36 & 51.26 & 56.76 & 63.29 & 56.82 \\
GPT-NeoX-20B & 71.42 & 68.24 & 69.67 & 64.51 & 65.65 & 64.96 & 52.24 & 52.18 & 63.61 \\
\midrule
\rowcolor{gray!20}
DetectRouter & \textbf{93.24} & \textbf{90.37} & \textbf{88.52} & \textbf{76.66} & 72.55 & \textbf{74.65} & \textbf{67.56} & 59.81 & \textbf{77.92} \\
\bottomrule
\end{tabular}
}
\end{table*}

\subsection{Prototype Space Geometry}
\label{app:analysis}

To verify the quality of the learned metric space, we visualize test set embeddings using t-SNE in Figure~\ref{fig:tsne_appendix}. The visualization reveals two critical geometric properties. First, texts generated by different model families form distinct, well-separated clusters, validating that our contrastive learning objective captures model-specific artifacts. Second, the learned prototypes are accurately positioned at the centroids of their respective source clusters. GPT-J and GPT-Neo clusters appear adjacent, reflecting their shared architectural lineage, yet remain distinguishable. This coherent structure confirms that DetectRouter learns a semantic manifold where proximity correlates with detector affinity, enabling robust generalization.

\begin{figure}[t]
    \centering
    \includegraphics[width=0.7\linewidth]{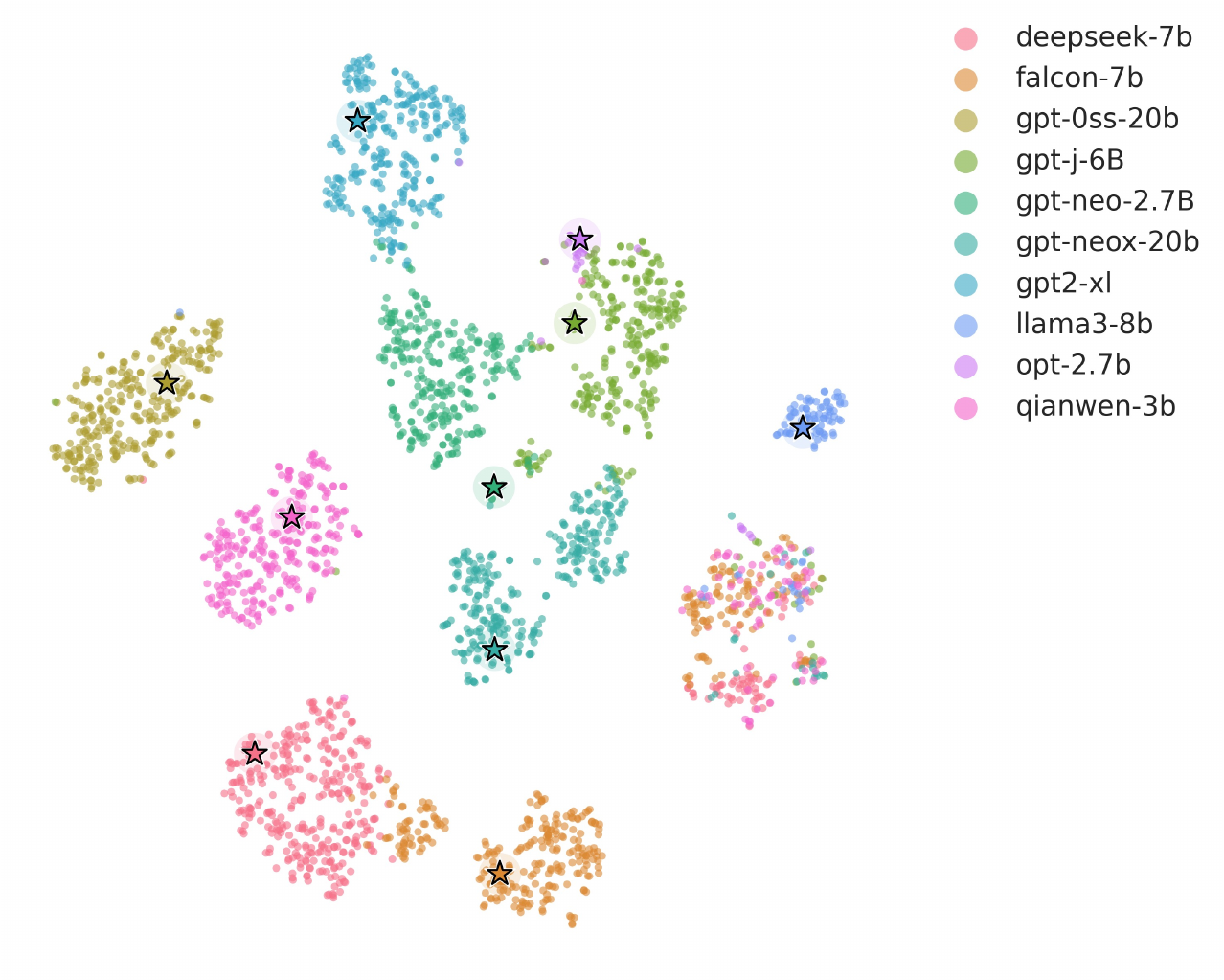}
    \caption{T-SNE projection of the embedding space showing text samples from different source models alongside their corresponding prototypes.}
    \label{fig:tsne_appendix}
\end{figure}

\section{Additional Analysis Visualizations}
\label{app:analysis_figures}

This section provides additional visualizations that complement the analysis in Section~\ref{sec:analysis}, illustrating the fundamental challenges that motivate our routing approach.

\subsection{Performance Gap Between Matched and Mismatched Settings.}
\label{app:analysis_white}
Figure~\ref{fig:analysis_gap} visualizes the performance degradation when moving from white-box to cross-model detection settings. For each of the five zero-shot detection methods, we compare the AUC achieved when the surrogate matches the source model versus when the surrogate is mismatched. The cyan bars represent white-box performance where the detector has access to the true source model, while the pink bars show the average cross-model performance with error bars indicating variance across different surrogate choices. The results reveal a consistent and substantial performance drop across all methods. Likelihood-based detection drops from 0.81 to 0.47, falling below random chance. LogRank exhibits the most dramatic degradation, plummeting from 0.85 to 0.50. Even Fast-DetectGPT, the most robust method, experiences a 20-point decline from 0.94 to 0.74. The large error bars in the cross-model setting further highlight the instability of fixed-surrogate approaches, where performance varies dramatically depending on the arbitrary choice of surrogate.

\begin{figure}[t]
    \centering
    \includegraphics[width=0.9\linewidth]{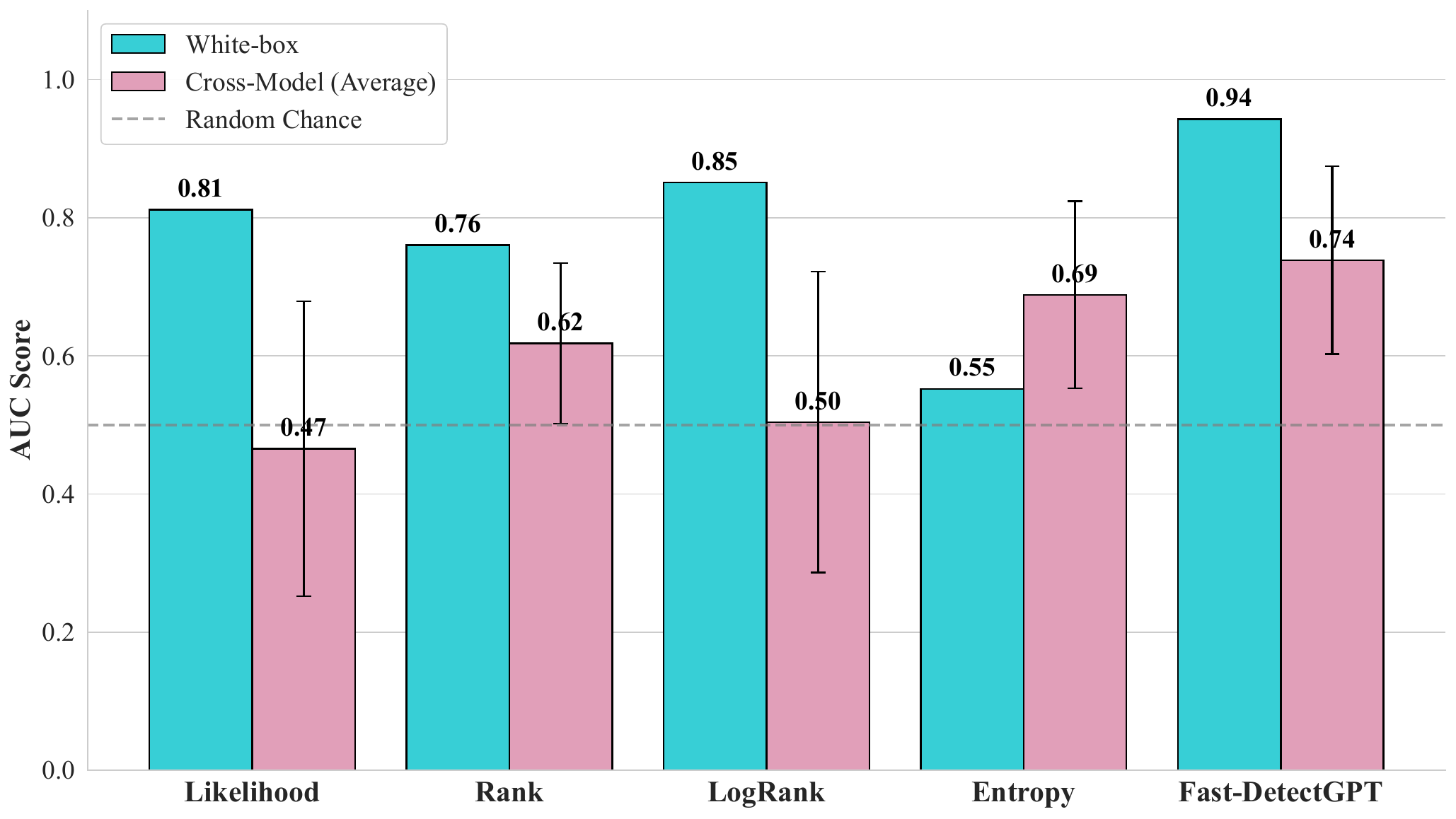}
    \caption{The performance gap across detection methods. Comparison of white-box versus cross-model detection performance across five zero-shot metrics on XSum. The drastic drop in mean AUC and the increase in variance for the cross-model setting demonstrate that fixed surrogates systematically fail to generalize. This visualization corresponds to the quantitative analysis in Table~\ref{tab:analysis_cost}.}
    \label{fig:analysis_gap}
\end{figure}

\subsection{Variance in Black-Box Detection Performance.}
\label{app:analysis_black}
Figure~\ref{fig:blackbox_variance} provides a detailed examination of detection performance variance when targeting black-box source models. Using the Fast-DetectGPT criterion on the MIRAGE benchmark, we evaluate each of our ten open-source surrogates against text generated by nine proprietary black-box model families: Claude, DeepSeek, Doubao, GPT, Llama, Qwen, Gemini, Grok, and Moonshot. Each column represents a different black-box source, with individual data points showing the AUC achieved by each surrogate. The box plots summarize the distribution of surrogate performance for each source. Two critical observations emerge from this visualization. First, optimal surrogates exist for every black-box source, with maximum AUC values ranging from 0.92 for Claude to 1.00 for Llama, demonstrating that effective detection is achievable given the right surrogate selection. Second, the performance variance is extreme, with AUC spreads exceeding 0.70 for several sources. DeepSeek exhibits the widest range from 0.15 to 0.99, while GPT spans from 0.30 to 0.96. This variance confirms that surrogate selection is not merely an optimization but a fundamental requirement for reliable detection. A random surrogate choice frequently yields near-random performance, as indicated by values approaching the 0.5 baseline shown by the red dashed line.

\begin{figure}[t]
    \centering
    \includegraphics[width=0.9\linewidth]{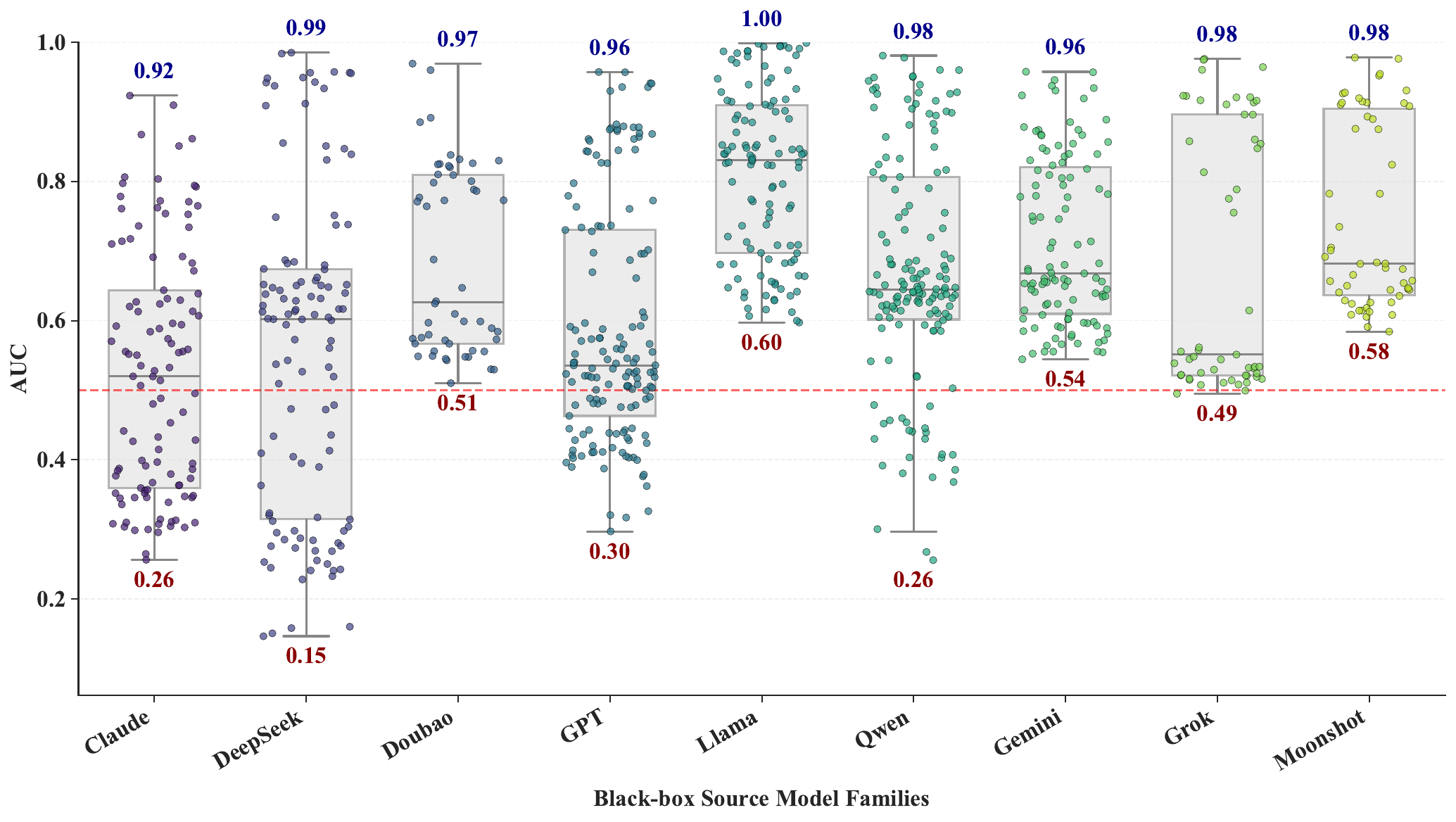}
    \caption{Detailed black-box performance variance. Evaluations on the MIRAGE benchmark~\citep{fu2025detectanyllm} using the Fast-DetectGPT criterion. Each data point represents the AUC achieved by a specific open-source surrogate when detecting text from a black-box generator family. The large vertical spread within each column reveals that for any unknown source, there exists an optimal surrogate in our pool achieving high AUC, but a random choice is likely to fail with AUC near the red dashed baseline. This extreme variance motivates the need for adaptive surrogate selection.}
    \label{fig:blackbox_variance}
\end{figure}

\section{Experimental Setup Details}
\label{app:exp_setup}

This section provides comprehensive details on training data construction, model configurations, and evaluation protocols.

\subsection{Training Data Construction}

\paragraph{Dataset Overview.}
Our training is based on the MIRAGE benchmark~\citep{fu2025detectanyllm}, the most comprehensive multi-task machine-generated text detection evaluation framework to date. MIRAGE samples human-written text from ten corpora across five domains: Academic from ArXiv and NeurIPS, News from CNN DailyMail and XSum, Comment from Amazon Review, Email from Enron Emails, and Website from OpenWebText. The benchmark employs 17 mainstream LLMs comprising 13 proprietary and 4 open-source models, and defines three generation tasks: \textit{Generate} for creating new text based on the first 30 tokens of human-written text, \textit{Polish} for refining existing text while preserving details and meaning, and \textit{Rewrite} for paraphrasing without altering semantic content. All texts are cleaned by removing special characters and filtered to 90--220 words to control for length-based detection biases.

\paragraph{Stage 1: Prototype Training Data.}
For discriminative prototype construction, we synthesize router-aligned training data based on the MIRAGE-DIG subset with disjoint-input generation. For each entry, we treat the original field as ground-truth human text and prompt locally cached language models to produce outputs for the Generate, Rewrite, and Polish tasks. The resulting records follow the same JSON schema as the source data for seamless integration into the existing pipeline. We generate 10,000 samples per task category, yielding 30,000 training instances in total.

\paragraph{Stage 2: Distance Regression Data.}
For black-box generalization training, we directly utilize the MIRAGE benchmark with detection scores annotated by running each text through our detector pool. Each sample is scored by all combinations of surrogate models and detection criteria, producing a score vector that serves as the regression target for learning text-detector affinity patterns.

\subsection{Evaluation Benchmarks}

We evaluate on two representative in-the-wild benchmarks that together assess cross-model generalization and temporal robustness.

\paragraph{MAGE Benchmark.}
MAGE~\citep{li2024mage} provides broad coverage across five domains including QA, Comment, News, Academic, and Story with diverse source models, representing established evaluation protocols in machine-generated text detection. Table~\ref{tab:mage_stats} summarizes the test set organization by source model family. The benchmark comprises 30,265 test pairs spanning eight model families, enabling comprehensive evaluation of cross-model generalization.

\begin{table}[t]
\centering
\caption{MAGE benchmark test set statistics by model family.}
\label{tab:mage_stats}
\small
\begin{tabular}{lr}
\toprule
Model Family & Test Pairs \\
\midrule
OpenAI & 8,180 \\
OPT & 7,956 \\
FLAN & 4,629 \\
LLaMA & 3,684 \\
BigScience & 2,672 \\
EleutherAI & 1,436 \\
GLM & 913 \\
Human Paraphrase & 795 \\
\midrule
Total & 30,265 \\
\bottomrule
\end{tabular}
\end{table}

\paragraph{EvoBench Benchmark.}
EvoBench~\citep{yu2025evobench} specifically targets the evolving nature of LLMs, capturing both updates from version changes by publishers and developments from optimizations by developers such as fine-tuning and pruning. GPT-4o, for example, receives updates approximately every three months. The benchmark covers seven LLM families across 29 evolving versions, with five domain datasets spanning XSum, Writing, PubMed, SocialMedia, and PeerRead across continuation-based, question-answering, and paraphrase generation paradigms. This benchmark directly tests whether routing mechanisms maintain robustness as source models evolve, a critical requirement for real-world deployment. Table~\ref{tab:evobench_stats} presents the test set statistics.

\begin{table}[t]
\centering
\caption{EvoBench benchmark test set statistics by model family. Datasets indicates the number of model-version and domain combinations.}
\label{tab:evobench_stats}
\small
\begin{tabular}{lrr}
\toprule
Model Family & Datasets & Test Pairs \\
\midrule
LLaMA-2-7B & 74 & 11,100 \\
GPT4o & 35 & 5,250 \\
Claude & 30 & 4,500 \\
GPT4 & 28 & 4,200 \\
LLaMA3 & 28 & 4,200 \\
Gemini & 15 & 2,243 \\
Qwen & 10 & 1,500 \\
\midrule
Total & 220 & 32,993 \\
\bottomrule
\end{tabular}
\end{table}

\paragraph{Evaluation Metric.}
Following standard practice in machine-generated text detection~\citep{mitchell2023detectgpt,bao2024fastdetectgpt,fu2025detectanyllm}, we adopt the Area Under the Receiver Operating Characteristic Curve as our primary evaluation metric. AUROC measures the probability that a randomly chosen machine-generated text receives a higher detection score than a randomly chosen human-written text, providing a threshold-independent assessment of detector discriminability. This metric ranges from 0.5 for random guessing to 1.0 for perfect separation.

\subsection{Surrogate Model Pool}
\label{app:models}

Our white-box surrogate pool comprises ten open-source language models spanning diverse architectures, training corpora, and parameter scales from 1.5B to 20B parameters.

DeepSeek-7B~\citep{deepseek2023llm} is a 7B parameter dense transformer optimized for efficient inference. Falcon-7B~\citep{almazrouei2023falcon} contains 7B parameters and was trained on RefinedWeb with multi-query attention. GPT-oss-20B~\citep{openai2025gptoss20b} is a 20B parameter open-source GPT variant with standard architecture. GPT-J-6B~\citep{wang2021gptj} is an autoregressive transformer with 6B parameters trained on The Pile by EleutherAI. GPT-Neo-2.7B~\citep{black2021gptneo} is a 2.7B parameter EleutherAI model with local attention patterns. GPT-NeoX-20B~\citep{biderman2023gptneox20b} contains 20B parameters with rotary position embeddings and parallelized attention. GPT2-XL~\citep{radford2019language} is the 1.5B parameter OpenAI baseline trained on WebText. LLaMA3-8B~\citep{dubey2024llama3} is an 8B parameter Meta instruction-tuned model with grouped-query attention. OPT-2.7B~\citep{zhang2022opt} is a 2.7B parameter Meta open pre-trained transformer. Qwen-3B~\citep{bai2023qwen} is a 3B parameter Alibaba multilingual model.

This diversity ensures that learned prototypes capture a broad spectrum of generation patterns across different model families, scales, and architectural choices.

\subsection{Detection Methods}

We evaluate six zero-shot detection criteria, each representing a different statistical perspective on distinguishing machine-generated from human-written text.

Likelihood~\citep{solaiman2019release} computes the average log-probability of tokens under the surrogate model, based on the intuition that machine-generated text tends to have higher likelihood under language models. Entropy measures token-level predictive uncertainty, which is typically higher for human text due to its greater unpredictability. Rank~\citep{gehrmann2019gltr} computes the average vocabulary rank of generated tokens, where machine-generated text tends to use higher-ranked tokens. LogRank~\citep{mitchell2023detectgpt} applies logarithmic transformation to token ranks for improved stability. LLR~\citep{su2023detectllm} computes the log-likelihood ratio between surrogate and reference models to capture relative probability differences. Fast-DetectGPT~\citep{bao2024fastdetectgpt} approximates conditional probability curvature without requiring expensive perturbations.

For each criterion, we compare the best fixed-surrogate baseline against our routed variant, which dynamically selects the optimal surrogate per input.

\subsection{Router Encoder Architectures}

We explore three encoder architectures for the routing network, trading off between computational efficiency and representational capacity.

RoBERTa-base~\citep{liu2019roberta} is our default choice with 125M parameters in a 12-layer Transformer architecture, offering efficient inference at 4.87ms per sample. RoBERTa-large provides stronger representations with 355M parameters in a 24-layer Transformer. LLM2Vec-LLaMA3-8B~\citep{behnamghader2024llm2vec} adapts the LLaMA3-8B decoder-only LLM for dense embeddings via masked next-token prediction and unsupervised contrastive learning, achieving the best overall performance at the cost of 166.2ms per sample.

\subsection{Hyperparameter Configuration}

\paragraph{Prototype Settings.}
We configure 10 learnable prototypes per class, using cosine similarity with temperature $\tau=10$ and L2-normalized features. Prototypes are initialized via Xavier initialization. We apply inter-class separation with margin 0.5 and weight $10^{-3}$, plus prototype norm regularization with weight $10^{-4}$.

\paragraph{Stage 1 Training.}
We train for 8 epochs with batch size 32, learning rate $2\times10^{-5}$, weight decay 0.01, and 200 warmup steps. The loss consists solely of cross-entropy classification loss.

\paragraph{Stage 2 Training.}
We train for 8 epochs with batch size 32, reduced learning rate $10^{-5}$, weight decay 0.01, and 200 warmup steps. The total loss combines cross-entropy loss with weight 1.0, distance regression loss with weight 1.0, and prototype anchoring loss with weight 1.0 to prevent manifold collapse during transfer.

\section{Analysis Experimental Setup}
\label{app:analysis_setup}

This section provides detailed experimental configurations for the surrogate-source affinity analysis presented in Section~\ref{sec:analysis}.

\subsection{White-box Cross-Evaluation Protocol}

\paragraph{Dataset Construction.}
Following the protocol established by Fast-DetectGPT~\citep{bao2024fastdetectgpt}, we use the XSum dataset~\citep{Narayan2018DontGM} as our seed corpus. XSum contains 226,711 news articles from BBC with single-sentence summaries. We randomly sample 500 articles from the test split and truncate each to its first 30 tokens as generation prompts. Each of the nine surrogate models then generates continuations of 150--200 tokens using nucleus sampling with $p=0.96$ and temperature $T=1.0$, matching the original Fast-DetectGPT configuration. This yields nine parallel corpora of machine-generated text, each containing 500 samples.

\paragraph{Surrogate Model Pool.}
The white-box analysis employs nine open-source LLMs spanning diverse architectures and scales: GPT-Neo-2.7B~\citep{black2021gptneo}, GPT-J-6B~\citep{wang2021gptj}, GPT-NeoX-20B~\citep{biderman2023gptneox20b}, OPT-2.7B~\citep{zhang2022opt}, LLaMA3-8B~\citep{dubey2024llama3}, Falcon-7B~\citep{almazrouei2023falcon}, DeepSeek-7B~\citep{deepseek2023llm}, Qwen-3B~\citep{bai2023qwen}, and GPT2-XL~\citep{radford2019language}. Parameter counts range from 1.5B to 20B, covering both earlier architectures like GPT-2 and GPT-Neo families as well as more recent models like LLaMA3 and DeepSeek.

\paragraph{Cross-Evaluation Matrix.}
We construct a $9 \times 9$ evaluation matrix where each row represents a source model and each column represents a surrogate model. For each source-surrogate pair, we compute detection scores on the 500 machine-generated samples from the source model using the surrogate model, paired with 500 human-written samples from the original XSum articles. We report AUROC for each cell, yielding 81 evaluation points that reveal the full landscape of cross-model detection performance.

\paragraph{Detection Criteria.}
The white-box analysis evaluates five zero-shot detection criteria: Likelihood, Entropy, Rank, LogRank, and Fast-DetectGPT. For each criterion, we construct the complete $9 \times 9$ affinity matrix. Figure~\ref{fig:analysis_confusion} in the main paper visualizes the Fast-DetectGPT matrix; additional matrices for other criteria exhibit similar diagonal-dominant patterns with criterion-specific variations.

\subsection{Black-box Evaluation Protocol}

\paragraph{Benchmark Selection.}
For black-box analysis, we utilize the MIRAGE benchmark~\citep{fu2025detectanyllm}, which provides machine-generated text from major commercial LLMs including GPT-4, GPT-4o, Claude, Gemini, DeepSeek, Grok, LLaMA, Qwen, and GLM. This benchmark covers five domains and three generation tasks, offering comprehensive coverage of real-world black-box scenarios.

\paragraph{Surrogate Sweep.}
We evaluate each of the nine open-source surrogate models on all black-box source families in MIRAGE. For each source-surrogate combination, we compute AUROC using the Fast-DetectGPT criterion. This yields a performance distribution for each source family, from which we extract the minimum, median, and maximum AUROC across surrogates.

\paragraph{Performance Gap Computation.}
The performance gap reported in Figure~\ref{fig:intro} is computed as the relative difference between the best and worst surrogate for each source family: $\text{Gap} = (\text{Max} - \text{Min}) / \text{Max} \times 100\%$. This metric quantifies the potential improvement from optimal surrogate selection compared to an unfortunate random choice.

\subsection{Quantitative Summary Statistics}

Table~\ref{tab:analysis_cost} in the main paper aggregates the cross-evaluation results into summary statistics. For the white-box setting, we report the average diagonal AUROC as the matched condition and the average off-diagonal AUROC as the cross condition. For the black-box setting, we report the mean AUROC across all source-surrogate pairs, along with the average of per-source maximum and minimum AUROC values. These statistics are computed separately for each detection criterion to reveal criterion-specific sensitivity to surrogate selection.

\section{Limitations}
\label{app:limitations}

We discuss several limitations of the proposed approach.

First, DetectRouter introduces additional computational overhead during inference. The router requires encoding each input text and computing distances to all prototypes before selecting the optimal detector. While lightweight encoders such as RoBERTa-base achieve acceptable latency at 4.87ms per sample, the best-performing LLM2Vec encoder incurs 166.2ms per sample, which may limit applicability in high-throughput scenarios.

Second, the effectiveness of our routing framework depends on the coverage of the surrogate model pool. When encountering text from a source model that is distributionally distant from all surrogates in the pool, the router may fail to identify a well-matched detector. Expanding the surrogate pool can mitigate this issue but increases both training and inference costs.

Third, Stage 2 training requires detection scores from diverse black-box sources as supervision. Collecting such annotations demands running multiple detectors across large-scale datasets, which incurs substantial computational cost during the data preparation phase.

Finally, our current evaluation focuses on English text. The generalization of DetectRouter to multilingual settings remains unexplored and may require language-specific surrogate pools or cross-lingual adaptation strategies.



\end{document}